\documentclass{article}

\PassOptionsToPackage{numbers, compress}{natbib}
\usepackage[preprint]{neurips_2026}

\usepackage[utf8]{inputenc} 
\usepackage[T1]{fontenc}    
\usepackage{hyperref}       
\usepackage{url}            
\usepackage{booktabs}       
\usepackage{amsfonts}       
\usepackage{nicefrac}       
\usepackage{microtype}      
\usepackage{xcolor}         
\usepackage{mathtools}
\usepackage{physics}
\usepackage{bm}
\usepackage{bbm}
\usepackage{comment}
\usepackage[skip=2pt]{caption}
\usepackage{graphicx}
\usepackage{caption}
\usepackage{subcaption}
\usepackage{booktabs}
\usepackage{amsthm}
\usepackage{amssymb}
\usepackage{placeins}

\usepackage[inline]{enumitem}
\setlist{nosep}

\newtheorem{theorem}{Theorem}
\newtheorem{assumption}{Assumption}
\newtheorem{definition}{Definition}
\newtheorem{remark}{Remark}
\newtheorem{proposition}{Proposition}
\newtheorem{lemma}{Lemma}

\newcommand{\R}{\mathbb{R}}

\newcommand{\N}{\mathcal{N}}

\newcommand{\mcE}{\mathcal{E}}

\newcommand{\mcD}{\mathcal{D}}

\newcommand{\mcR}{\mathcal{R}}
\newcommand{\xvec}{\bm{x}}
\newcommand{\yvec}{\bm{y}}
\newcommand{\zvec}{\bm{z}}
\newcommand{\mvec}{\bm{m}}

\newcommand{\wvec}{\bm{w}}

\newcommand{\gvec}{\bm{g}}
\newcommand{\svec}{\bm{s}}

\newcommand{\vvec}{\bm{v}}
\newcommand{\thetavec}{\bm{\theta}}
\newcommand{\phivec}{\bm{\phi}}

\newcommand{\rhovec}{\bm{\rho}}
\newcommand{\chivec}{\bm{\chi}}

\newcommand{\Gmat}{\bm{G}}
\newcommand{\Hmat}{\bm{H}}
\newcommand{\Imat}{\bm{I}}
\newcommand{\Jmat}{\bm{J}}

\newcommand{\Lmat}{\bm{L}}

\newcommand{\Tmat}{\bm{T}}

\newcommand{\Sigmamat}{\bm{\Sigma}}

\DeclareMathOperator*{\argmin}{arg\,min}

\title{Joint 3D Gravity and Magnetic Inversion via Rectified Flow and Ginzburg-Landau Guidance}

\author{%
  Dhruman Gupta \\
  Department of Computer Science \\
  Ashoka University \\
  \texttt{dhurman.gupta\_ug2023@ashoka.edu.in}
  \And
  Yashas Shende \\
  Department of Physics \\
  Ashoka University \\
  \texttt{yashas.shende\_ug25@ashoka.edu.in}
  \AND
  Aritra Das \\
  Department of Computer Science \\
  Ashoka University \\
  \texttt{aritra.das@ashoka.edu.in}
  \And
  Chanda Grover Kamra \\
  Department of Computer Science \\
  Ashoka University \\
  \texttt{chanda.grover\_phd19@ashoka.edu.in}
  \AND
  Debayan Gupta \\
  Department of Computer Science \\
  Ashoka University \\
  \texttt{debayan.gupta@ashoka.edu.in}
}

\begin{document}

\maketitle

\begin{abstract}
  Subsurface ore detection is of paramount importance given the rising depletion of shallow mineral resources in recent years. It is crucial to explore approaches that go beyond the limitations of traditional geological exploration methods. Due to readily available surface readings, joint magnetic and gravitational inversion is a promising new method -- given magnetic and gravitational data on a surface, jointly reconstructing the underlying densities that generate them. However, this is ill-posed and has non-unique solutions. Deterministic methods often require handcrafted priors and converge to a single solution and do not capture the distribution, which is often of interest. We introduce a novel framework that reframes 3D gravity and magnetic joint inversion as a rectified flow on the Noddyverse dataset, the largest physics-based dataset for inversion. We introduce a Ginzburg-Landau (GL) regularizer, a generalized version of the Ising model that aids in ore identification, enabling physics-aware training. We also propose a guidance methodology based on GL theory that can be used as a plug-and-play module with existing unconditional denoisers. Lastly, we also train and release a VAE for the 3D densities, which facilitates downstream work in the field.

\end{abstract}

\section{Introduction}
\label{sec:intro}

Gravitational and magnetic surveys are among the most widely deployed geophysical measurements for regional mapping and mineral exploration \cite{blakely1995potential,hinze2013gravity}.
However, turning these surface measurements into 3D subsurface property models is fundamentally challenging~\cite{li1996mag3d}. Inversion of gravitational and magnetic potential fields is ill-posed and non-unique, with many distinct subsurface distributions producing indistinguishable gravitational and magnetic fields~\cite{tarantola2005inverse,tikhonov1977illposed}.

Classical methods recover a single regularized estimate by balancing data misfit with hand-crafted priors \cite{li1996mag3d,li1998grav3d,portniaguine1999focusing}, not accounting for the probabilistic nature of the solution, which is critical for proper inversion. Probabilistic methods such as Monte Carlo sampling \cite{mosegaard1995montecarlo} are computationally prohibitive at the scale of modern datasets \cite{tarantola2005inverse}. Furthermore, incorporating priors that encode both geological realism and sharp host-ore boundaries is challenging for standard probablistic and classical methods.

ML approaches allow the encoding of realistic geology~\cite{negahdari2025performancemachinelearningmethods}. However, most ML models for magnetic and gravity inversion are trained on simplistic data ~\cite{huang2021deep,rs16060995}. This disconnect between training data and real-world geology limits the reliability of such approaches in practice. The Noddyverse dataset is a large-scale physics-based dataset of synthetic geological settings \cite{jessell2022noddyverse}. This makes it a relevant dataset for training ML models on geophysically realistic data, enabling models to learn geological structures and physical relationships that better reflect the true subsurface.


Diffusion and flow-based generative models have emerged as powerful tools for sampling from complex data distributions \cite{ho2020ddpm,Liu2022RectifiedFlow}. These can be adapted to inverse problems by combining the learned priors with likelihood terms \cite{Chung2023DPS,kim2025flowdpsflowdrivenposteriorsampling}. Furthermore, performing this in a compressed latent space can be substantially more efficient than in the original data space \cite{rombach2022ldm}. This leads to the following question: \textit{Can joint gravity and magnetic inversion produce scalable posterior samples of realistic 3D ore models with sharp boundaries?}

We approach this by combining (i) a learned latent generative prior trained on Noddyverse with (ii) a novel physics-inspired Ginzburg--Landau (GL) prior that promotes phase separation and interface regularity. Our main contributions are:
\begin{itemize}
    \item \textbf{Joint potential-field inversion as latent posterior sampling.}
    We reframe 3D joint gravity and magnetic inversion as posterior sampling with a latent generative prior, enabling stochastic reconstruction in a scalable space.
    \item \textbf{GL regularization for ore-aware structure.}
    We introduce a Ginzburg--Landau regularizer that encodes phase-separation structure and integrate it into flow-based inversion as a physics-aware regularizer.
    \item \textbf{Plug-and-play GL guidance.}
    We propose a GL guidance step that can be combined with existing unconditional denoisers in a modular fashion.
    \item \textbf{Benchmark assets for the community.}
    We train and release a 3D VAE for density/susceptibility volumes and provide an efficient pre-processing pipeline for Noddyverse-style data to facilitate downstream work.
    \item \textbf{First flow-based study for joint magnetic and gravitational inversion.} To the best of our knowledge, we are the first to develop a generative model for inversion over magnetic and gravitational fields.
\end{itemize}

\noindent
All code is publicly available at:\\
\url{https://anonymous.4open.science/r/maggrav-joint-diffusion-DB71}.


\section{Inverse Problems}
\label{sec:inversion_problems}

Inverse problems aim to recover hidden quantities from indirect measurements. In the simplest setting, observations are generated by
\begin{equation}
    y = \mathcal{A}(x) + \epsilon,
    \label{eq:generic_inverse_problem}
\end{equation}
where $x$ is the unknown model, $y$ is the observed data, $\mathcal{A}$ is the forward operator, and $\epsilon$ denotes measurement noise. The inverse problem attempts to infer $x$ from $y$.

This inference is difficult because many inverse problems are ill-posed: a solution may fail to exist, may not be unique, or may depend unstably on small perturbations in the data~\cite{tikhonov1977illposed,tarantola2005inverse}. Potential-field inversion is a canonical example. Gravity and magnetic fields measured at the surface are spatially smoothed responses of subsurface density and susceptibility, so high-frequency or deeply buried structures can be weakly constrained by the data. As a result, distinct subsurface models can produce nearly identical observations~\cite{saltus2011nonunique}.

Classical inversion addresses this by regularization. For a linear forward model with Gaussian noise, a common formulation is
\begin{equation}
    \hat{x}
    =
    \argmin_x
    \frac{1}{2\sigma_y^2}\|y-Ax\|^2
    +
    \lambda R(x),
    \label{eq:regularized_inverse_problem}
\end{equation}
where the first term enforces data consistency and $R(x)$ encodes prior assumptions such as smoothness, sparsity, compactness, or structural coherence. However, the result can be sensitive to the selected regularizer and hyperparameters, and it does not by itself characterize the range of plausible models consistent with the same data.

A probabilistic view instead treats $x$ as a random variable and targets the posterior distribution
\begin{equation}
    p(x \mid y) \propto p(y \mid x)p(x),
    \label{eq:bayesian_inverse_problem}
\end{equation}
where $p(y \mid x)$ is determined by the forward model and noise statistics, and $p(x)$ represents prior knowledge about plausible subsurface structure. Sampling from this posterior is especially valuable in geophysical settings because uncertainty and ambiguity are often as important as a single reconstruction. The challenge is that realistic 3D models are high-dimensional, and traditional sampling methods can be computationally prohibitive at this scale~\cite{mosegaard1995montecarlo}.

In this work, the unknown model is the joint density--susceptibility volume $m=(\rho,\chi)$ and the observations are gravity and magnetic surface fields. The corresponding forward model is linear after discretization and is written explicitly in Section~\ref{sec:theory} as $y=\Gmat m+\epsilon$. Our goal is therefore to sample realistic 3D ore models that are consistent with the observations.

\section{Related Work}
\label{sec:litreview}

\noindent\textbf{Potential-field inversion:}
Gravity and magnetic inversion are classic examples of ill-posed potential-field problems, where many subsurface density and susceptibility distributions can fit the same observations.
Voxel-based 3D inversion with Tikhonov-style regularization and depth weighting remains a traditional approach for both magnetics and gravity~\cite{LiOldenburg1996,LiOldenburg1998}.
To recover geologically plausible, compact bodies, regularizers such as edge-preserving or focusing stabilizers and minimum-gradient-support have been used.~\cite{LastKubik1983,PortniaguineZhdanov1999}.

\noindent\textbf{Joint gravity--magnetic inversion:}
Combining gravity and magnetic data can reduce non-uniqueness by exploiting their complementary sensitivities.
Early approaches used lower-dimensional or layered parameterizations ~\cite{GallardoDelgado2003,Pilkington2006}.
General 3D joint inversion introduces coupling terms that encode either structural similarity or shared geometry between recovered density and susceptibility models~\cite{FregosoGallardo2009,Zhdanov2012,Lin2018,Utsugi2025}.
Recent formulations extend these ideas to multiple potential-field data types while keeping memory and runtime manageable~\cite{Vatankhah2023}.

\noindent\textbf{Stochastic and uncertainty-aware methods:}
While regularized least-squares formulations return a single MAP-like model, probabilistic joint inversion has been pursued to characterize uncertainty, including Monte-Carlo-based gravity and magnetic inversion strategies that target posterior statistics~\cite{Bosch2006}.

\noindent\textbf{Machine Learning-based potential-field inversion:}
Supervised neural networks have been trained to directly map potential-field measurements to subsurface models, enabling rapid inversion once trained. Recent work applies deep networks to gravity inversion~\cite{huang2021deep}, joint gravity and magnetic inversion~\cite{Min2024GMNet}, and to Noddy-based magnetic inversion examples~\cite{Guo2021Noddy}.
However, most neural networks produce point estimates and do not explicitly represent posterior stochasticity.

\noindent\textbf{Generative models for inverse problems:}
Diffusion and score-based generative models have emerged as generative priors that support conditional generation and approximate posterior sampling.
Foundational work includes denoising diffusion models~\cite{ho2020ddpm} and the SDE score-based formalism~\cite{Song2021SDE}; latent diffusion improves efficiency by operating in an autoencoder latent space~\cite{rombach2022ldm}.
EDM preconditioning~\cite{Karras2022EDM} and rectified flow~\cite{Liu2022RectifiedFlow} enable faster sampling. Further, general inverse-problem solvers such as diffusion posterior sampling (DPS) and DDRM incorporate measurement consistency via iterative guidance~\cite{Chung2023DPS,Kawar2022DDRM}.

\section{Theoretical Background}
\label{sec:theory}

We now discuss the forward operators for gravitational and magnetic inversion.

\subsection{Gravitational and Magnetic Forward Operators}

\begin{definition}[Gravity Anomaly]
The vertical component of the gravitational acceleration anomaly is:
\begin{equation}
g_z(\xvec) = -\frac{\partial U}{\partial z} = G \int_V \rho(\xvec') \frac{z - z'}{|\xvec - \xvec'|^3} \, d^3\xvec'
\label{eq:gravity_anomaly}
\end{equation}
\end{definition}

$z$ is positive upward and we report $g_z$ as the downward component $-g_z^{(\text{up})}$.

In discretized form, $\gvec_{\text{grav}} = \Gmat_\rho \rhovec$. The gravity sensitivity kernel uses the standard analytical kernel for rectangular prisms (see Appendix~\ref{subsec:gravity_prism}).

The magnetic field anomaly arises from induced and remanent magnetization. For induced magnetization in a region with susceptibility $\chi(\xvec')$: $\bm{M}(\xvec') = \chi(\xvec') \bm{H}_0$ where $\bm{H}_0$ is the ambient geomagnetic field.

\begin{definition}[Total Magnetic Intensity Anomaly]
Total magnetic intensity (TMI) anomaly projected onto the direction of the ambient field $\hat{\bm{H}}_0$ is:
\begin{equation*}
\Delta T(\xvec) = -\frac{\mu_0 \hat{\bm{H}}_0}{4 \pi} \cdot \nabla \int_V \bm{M}(\xvec') \cdot \nabla' \frac{1}{|\xvec - \xvec'|} \, d^3\xvec'
\end{equation*}
\end{definition}

In discretized form, $\gvec_{\text{mag}} = \Gmat_\chi \chivec$. The magnetic sensitivity kernel depends on the geometry and ambient field direction (see Appendix~\ref{subsec:magnetic_kernel}).

\subsection{Joint Forward Model}

Combining gravity and magnetic observations into a unified framework:

\begin{equation}
\underbrace{\begin{pmatrix} \gvec_{\text{grav}} \\ \gvec_{\text{mag}} \end{pmatrix}}_{\yvec} = 
\underbrace{\begin{pmatrix} \Gmat_\rho & \bm{0} \\ \bm{0} & \Gmat_\chi \end{pmatrix}}_{\Gmat}
\underbrace{\begin{pmatrix} \rhovec \\ \chivec \end{pmatrix}}_{\mvec} + 
\underbrace{\begin{pmatrix} \bm{\epsilon}_g \\ \bm{\epsilon}_m \end{pmatrix}}_{\bm{\epsilon}}
\label{eq:joint_forward}
\end{equation}

We denote $\mvec = [\rhovec^\top, \chivec^\top]^\top \in \R^{2N}$ as the combined model parameter vector.

\begin{assumption}[Noise Model]
\label{assumption:noise}
Measurement noise $\bm{\epsilon}$ is assumed to be Gaussian: $\bm{\epsilon} \sim \N(\bm{0}, \Sigmamat_\epsilon)$
for some $\Sigmamat_\epsilon$.
\end{assumption}

\subsection{Integrating Physics}

In ore deposit modeling, we expect sharp boundaries between ore and host rocks, suggesting an underlying discrete structure. The Ising model naturally captures such binary behavior with spatial correlation:

\begin{equation*}
H_{\text{Ising}} = -J \sum_{\langle i,j \rangle} s_i s_j - h \sum_i s_i, \quad s_i \in \{-1, +1\}
\end{equation*}

However, discrete models pose challenges:
\begin{enumerate*}[label=\roman*), itemjoin={{; }}, after={{.}}]
    \item Non-differentiability prevents gradient-based optimization
    \item Incompatibility with continuous diffusion processes
    \item Difficulty in probabilistic modeling
\end{enumerate*} The Ginzburg-Landau (GL) functional gives a differentiable continuous relaxation that preserves the essential phase-separation physics.

\subsection{Ginzburg-Landau Energy for Continuous Ising-Style Dynamics}
\label{sec:ginzburg_landau}

\begin{definition}[Ginzburg-Landau Free Energy]
For a continuous order parameter field $\phi(\xvec): \Omega \to \R$ representing susceptibility (normalized):
\begin{equation}
\mcE_{\text{GL}}[\phi] = \int_\Omega \left[\frac{\kappa}{2}|\nabla\phi|^2 + W(\phi)\right] d\xvec
\label{eq:gl_functional}
\end{equation}
where $\kappa > 0$ is the gradient penalty coefficient (interface energy), and $W(\phi)=\frac{1}{4\epsilon^2}(\phi^2-1)^2$ is the double-well potential.
\end{definition}

\begin{remark}
  The GL free energy provides a continuous extension: the first term is minimized when $\nabla \phi = 0$, favoring spatially smooth regions and penalizing sharp interfaces, while the second term is minimized at $\phi=\pm1$, representing the ore/host phases.
\end{remark}

\noindent \noindent The GL energy acts as a differentiable extension for Ising-style binary phase structure, with a formal connection to perimeter-regularized binary models. See Appendix~\ref{subsec:ising_gl_correspondence} for the formal connection. We therefore use GL as a differentiable surrogate for ore/host phase separation.

\noindent For computation, we discretize \eqref{eq:gl_functional} on the voxel grid. In matrix form,
\begin{equation}
\mcE_{\text{GL}}[\phivec] = \frac{\kappa}{2h^2}\phivec^\top \Lmat \phivec + \frac{1}{4\epsilon^2}\sum_{i=1}^{N}(\phi_i^2 - 1)^2
\label{eq:gl_matrix}
\end{equation}
\noindent where $h$ is the grid spacing and $\Lmat$ is the standard graph Laplacian; see Appendix~\ref{subsec:discrete_gl}.

\begin{proposition}[GL Energy Gradient]
The gradient of the GL energy with respect to $\phivec$ is:
\begin{equation}
\nabla_{\phivec} \mcE_{\text{GL}} = \frac{\kappa}{h^2} \Lmat \phivec + \frac{1}{\epsilon^2}\phivec \odot (\phivec \odot \phivec - \bm{1})
\label{eq:gl_gradient}
\end{equation}
where $\odot$ denotes element-wise multiplication. Proof is in Appendix~\ref{subsec:gl_gradient_appendix}.
\end{proposition}

\noindent To apply GL regularization to susceptibility, we map $\chi \in [\chi_{\min}, \chi_{\max}]$ to the phase field
\begin{equation*}
\phi(\chi) = \frac{2\chi - (\chi_{\max} + \chi_{\min})}{\chi_{\max} - \chi_{\min}}.
\end{equation*}
The corresponding regularized inverse-problem energy combines data consistency and GL structure:
\begin{align*}
\mcE_{\text{data}}[\mvec] &= \frac{1}{2\sigma_y^2}\|\yvec - \Gmat\mvec\|^2 \\
\mcE_{\text{GL}}[\phi(\chi)] &= \frac{\kappa}{2h^2}\phivec^\top \Lmat \phivec + \frac{1}{4\epsilon^2}\|\phivec \odot \phivec - \bm{1}\|^2.
\end{align*}

\noindent We interpret the GL penalty as an implicit prior, $p_{\text{GL}}(\mvec) \propto \exp(-\lambda_{\text{GL}}\mcE_{\text{GL}}[\phi(\mvec)])$, whose score contributes $-\lambda_{\text{GL}}\nabla_{\mvec}\mcE_{\text{GL}}[\phi(\mvec)]$. Combining it with the learned prior gives

\begin{equation}
\svec_{\text{combined}}(\mvec, t) = \svec_{\theta}(\mvec, t) - \lambda_{\text{GL}}(t)\nabla_{\mvec} \mcE_{\text{GL}}[\phi(\mvec)],
\label{eq:combined_score}
\end{equation}

where $\lambda_{\text{GL}}(t)$ controls the time-dependent strength of GL guidance.
The same GL energy also induces Allen--Cahn phase-separation dynamics; we include this connection and its Gibbs interpretation in Appendix~\ref{subsec:allen_cahn}.

\section{Experimental Setup}
\label{sec:arch}
\subsection{Dataset and Pre-processing}
\label{sec:dataset}
We use Noddyverse~\citep{jessell2022noddyverse}, a physics-based synthetic benchmark for geophysical inversion. Each sample contains a 3D subsurface block of size $200 \times 200 \times 200$ voxels, where each voxel stores gravitational density ($\rhovec$) and magnetic susceptibility ($\chivec$). The corresponding observations are two forward-simulated 2D surface fields on a $200 \times 200$ grid: gravity anomaly ($\gvec_{\text{grav}}$) and total magnetic intensity ($\gvec_{\text{mag}}$). The learning task is to reconstruct 3D property volumes from  2D measurements.

Direct training on the raw files is difficult for three reasons. First, the target dimensionality is very large ($200^3$ voxels per property), making inversion highly underdetermined and memory-intensive. Second, magnetic susceptibility is strongly heavy-tailed, which can destabilize optimization. To make the dataset practical for ML workflows, we implement and release a three-stage pipeline:
\begin{enumerate}
    \item \textbf{Pre-processing.} We parse raw files into structured tensors, handle outliers, and optimize data loading for large $200^3$ volumes.
    \item \textbf{VAE training and tuning.} We train and tune a 3D VAE to learn a compact latent representation of coupled density--susceptibility volumes.
    \item \textbf{Generative model training.} We train latent diffusion and rectified flow models on the learned VAE representation.
\end{enumerate}

\subsection{VAE}
We train a compact 3D variational autoencoder to map paired density--susceptibility volumes \(x \in \R^{2 \times 200 \times 200 \times 200}\) to a structured latent grid. The encoder uses residual 3D convolutional downsampling with bottleneck attention and outputs the usual Gaussian parameters \((\mu,\log\sigma^2)\); the decoder mirrors this hierarchy to reconstruct the two-channel volume. The encoder bottleneck naturally produces a \(25^3\) feature grid, which we project to \(24^3\) for downstream compatibility with the latent diffusion and flow backbones.

\begin{figure}
  \centering
  \includegraphics[width=1\textwidth]{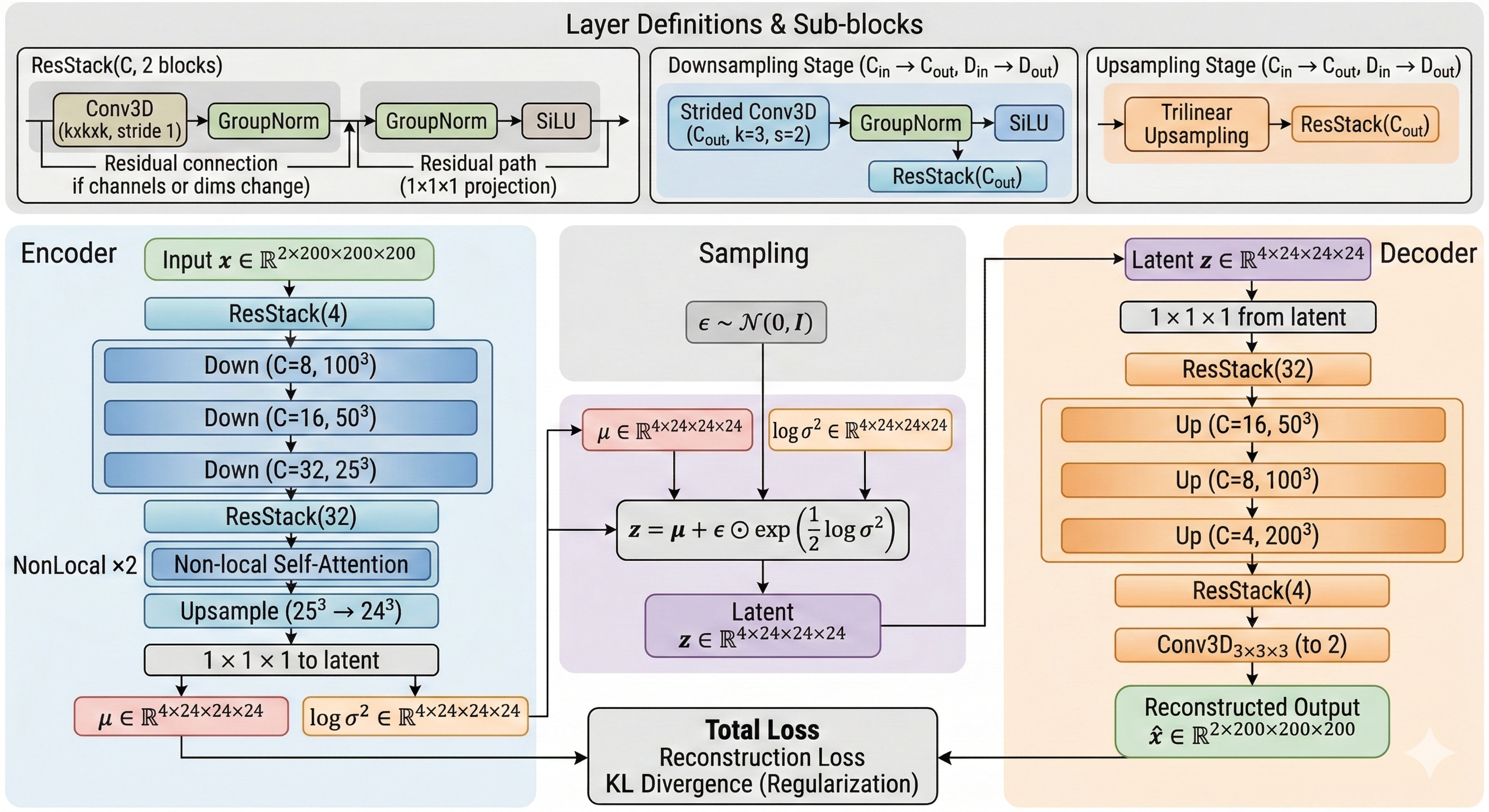}
  \caption{Overview of our VAE architecture (generated with AI).}
  \label{fig:vae-architecture}
\end{figure}

We performed hyperparameter sweeps over latent size, model capacity (parameter count), and KL weight \(\beta\). The final setting balances reconstruction quality, latent regularization, and compute/memory constraints: latent grid \(24^3\), \(\beta=10^{-3}\), and \(385{,}736\) parameters. A larger latent grid (\(40^3\)) yields substantially lower reconstruction loss, but we retain \(24^3\) for practical compute.

\subsection{Diffusion and Flow Model}
Two 3D latent-space backbones are used: a timestep-conditioned 3D UNet \cite{cicek2016unet} and BiFlowNet. \cite{wang20253dmeddiffusion3dmedical}  
The 3D UNet is built from residual convolutional blocks with FiLM-based time conditioning, symmetric downsampling/upsampling paths with skip connections, and self-attention at lower spatial resolutions. BiFlowNet adopts a hybrid multiscale design that couples a UNet-like 3D hierarchy with patch-level transformer-style (DiT-like) processing and cross-scale feature fusion.
Empirically, BiFlowNet demonstrates better parameter efficiency and faster optimization dynamics than the 3D UNet baseline. Training is conducted under two continuous-time formulations: EDM diffusion parameterization \cite{Karras2022EDM} and rectified flow~\cite{Liu2022RectifiedFlow}. Preliminary comparisons show faster convergence in terms of epochs for rectified flow, so the final model is trained with the flow parameterization.

Finally, our rectified-flow uses VAE-encoded latent volumes with timesteps sampled from \(t \sim \mathcal{U}(0,1)\), and the model is optimized to predict the linear-flow velocity field.

\subsection{Latent Posterior Sampling for Inversion}
\label{sec:dps}
We combine the learned generative prior with the GL structural prior to perform posterior sampling for joint gravity--magnetic inversion. Given observations \(y\), our goal is to sample subsurface models that both match the measured fields and exhibit physically plausible susceptibility structure:
\begin{equation}
p(m \mid y) \propto
p(y \mid \Gmat m)\,
p_\theta(m)\,
p_{\mathrm{GL}}(\chi).
\end{equation}

Here, \(m=(\rho,\chi)\) is the joint density--susceptibility model and \(\Gmat\) is the joint gravity--magnetic forward operator from \eqref{eq:joint_forward}. We combine the learned generative score with the GL score as in \eqref{eq:combined_score}, with the GL term acting only on the susceptibility channel.

This augmented score allows us to use diffusion/flow-based inversion techniques to sample from the desired posterior while steering susceptibility fields toward phase-separated ore/host structure.

In implementation, sampling is performed in the VAE latent space. Let \(\mathcal{D}\) denote the VAE decoder and let \(\mathcal{D}_\chi\) denote its susceptibility channel. The model-space posterior is instantiated as
\begin{equation}
p(z \mid y) \propto
p(y \mid \Gmat\mathcal{D}(z))\,
p_\theta(z)\,
p_{\mathrm{GL}}(\mathcal{D}_\chi(z)).
\end{equation}

This decoder-space formulation is necessary because GL guidance should act on physical susceptibility fields, not directly on latent coordinates. We therefore compute the GL gradient on the decoded susceptibility volume and pull it back through the decoder,
\begin{equation}
s_{\text{GL}}(z_t,t)
=
s_\theta(z_t,t)
- \lambda_t
\left(
\frac{\partial \mathcal{D}_{\chi}(z_t)}{\partial z_t}
\right)^{\!\top}
\nabla_{\mathcal{D}_{\chi}(z_t)}
\mathcal{E}_{\text{GL}}
.
\end{equation}

We then apply the standard FlowDPS update\cite{kim2025flowdpsflowdrivenposteriorsampling} to enforce consistency with the joint forward model. The inversion-specific data term is
\begin{equation}
\mathcal{L}_{\mathrm{data}}(m)=
\frac{1}{2}\,\Big\lVert \Gmat m-y \Big\rVert_{\Sigma^{-1}}^2,
\qquad
\Sigma=\mathrm{diag}(\sigma_{\mathrm{mag}}^2,\sigma_{\mathrm{grav}}^2),
\end{equation}
where the covariance accounts for the different scales of the magnetic and gravity observations. In latent-space sampling, this loss is evaluated on decoded models \(m=\mathcal{D}(z)\). The full FlowDPS endpoint refinement and propagation equations are given in Appendix~\ref{subsec:flowdps_update}.

\subsection{Training-Time GL Regularization}
\label{sec-limits}
GL structure can be introduced either during training, as a regularization term, or during inference, as posterior guidance. Concretely, training minimizes
\[
\mathcal{L}_{\text{GL-flow}}
=
\frac{\sum_i w_i\,\ell_i}{\sum_i w_i},
\:\:
w_i \propto \exp\!\left(-\lambda\,\mathcal{E}_{\mathrm{GL}}(\mathcal{D}_{\chi}(z_i))\right),
\]
where \(\ell_i\) is the per-sample rectified-flow loss and the GL energy is evaluated on the decoded susceptibility field. However, using this term directly would require decoding each latent sample to a full \(200^3\) susceptibility volume and differentiating the GL energy through the decoder at every optimization step, which is prohibitive in memory and compute.

Still, we experiment with using the direct latent grid in the loss term. As expected, since the latent space does not resemble the susceptibility space, this leads to worse-than-baseline performance compared with training without the term. We discuss the results for this in Appendix~\ref{sec:latent_gl_proxy_appendix}.

\subsection{Inference Hyperparameters and Stabilization}
For all posterior-sampling experiments, we use the following FlowDPS and guidance settings:
\[
N_{\text{steps}}=64,\:\:
K_{\text{ref}}=8,\;\:
\alpha_{\text{ref}}=0.1,\:\:
\sigma_{\mathrm{mag}}=15,\:\:
\sigma_{\mathrm{grav}}=0.1.
\]
Since magnetic response depends on exponentiated susceptibility, trajectories are highly perturbation-sensitive (\(\chi=10^{\chi_{\log_{10}}}-\epsilon\)); we therefore clamp the norm of the posterior guidance score for stability.

The GL guidance weight is time-dependent and follows \(\lambda_t=\lambda_0(1-t)^\gamma\) (implemented with clamped \(1-t\in[0,1]\)); this keeps guidance weak at high-noise stages to preserve exploration, then increases guidance in later denoising steps to steer samples toward physically plausible solutions.

\section{Results}
\label{sec:results}

Our DPS inversion produces plausible 3D posterior samples while matching magnetic and gravity observations. In the main results, we compare the baseline flow model against decoder-space GL guidance. The latent-space GL regularization proxy is reported separately in Appendix~\ref{sec:latent_gl_proxy_appendix}.

\begin{figure}[htbp]
  \centering
  \includegraphics[width=0.9\linewidth]{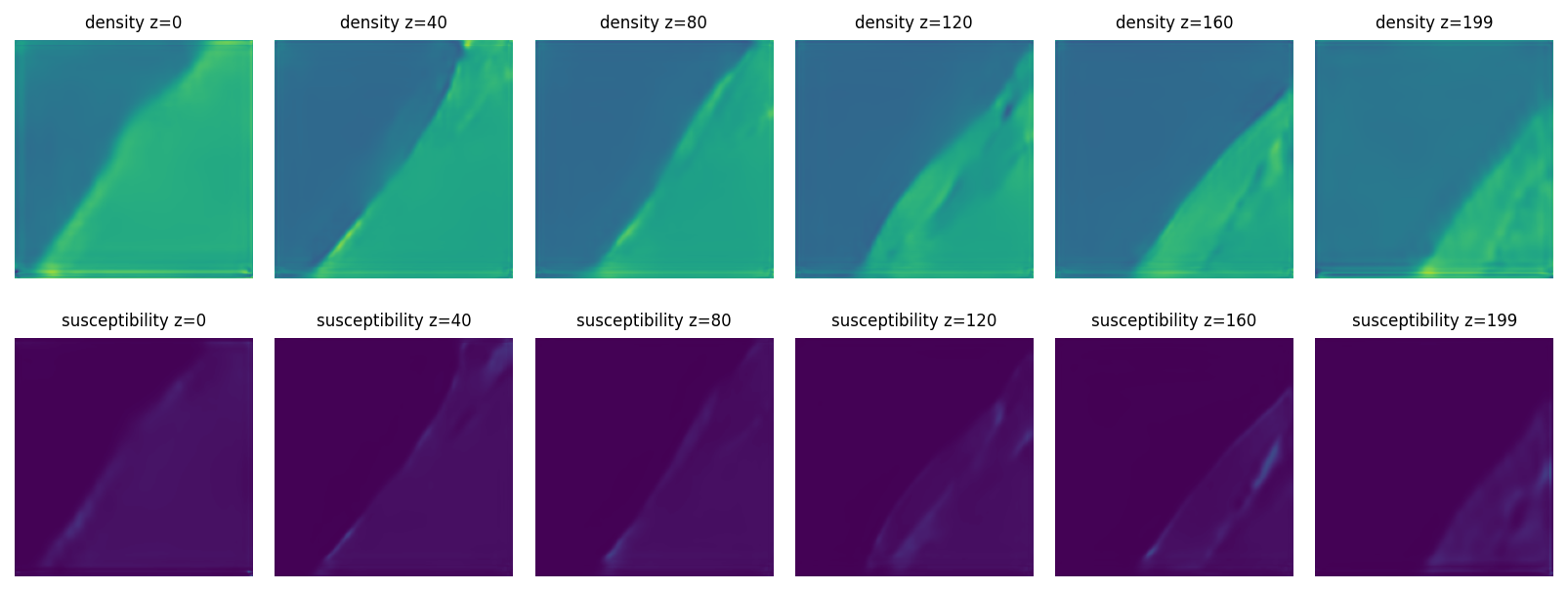}
  \caption{Example generated 3D volume visualized across multiple \(z\)-slices.}
  \label{fig:sample_3d_slices}
\end{figure}

Figure~\ref{fig:sample_3d_slices} shows representative slices from a generated 3D posterior sample. The model produces geologically plausible structures while remaining consistent with the observed fields. To complement this volumetric view, we visualize observation-conditioned forward responses in Figure~\ref{fig:best_sample_grids}. For each selected observation, we plot observed fields alongside baseline and GL-guided predictions. The GL-guidance variant is visually more consistent with observations across both modalities.

\begin{figure}[htbp]
  \centering
  \begin{subfigure}[t]{0.49\textwidth}
    \centering
    \includegraphics[width=\linewidth]{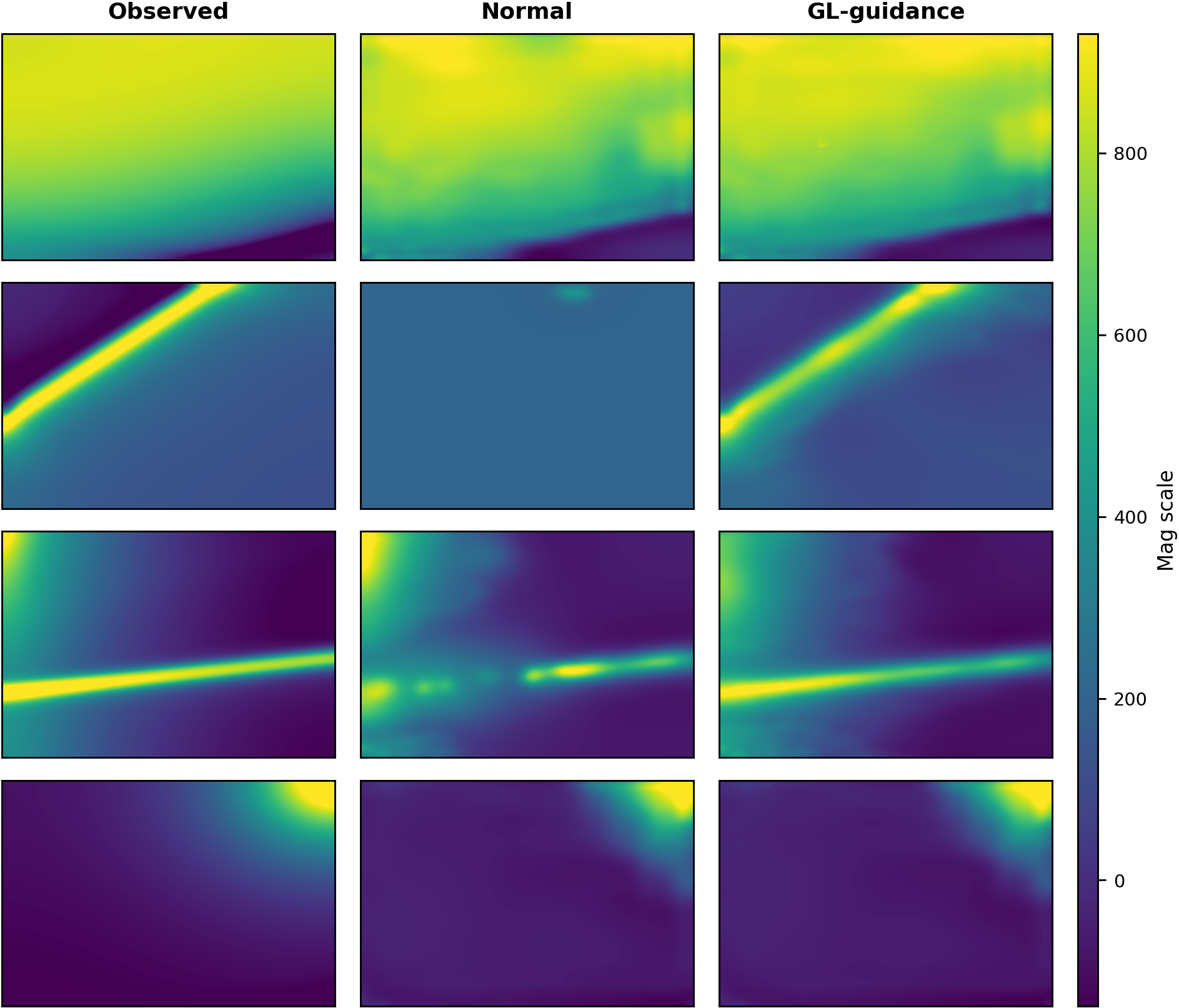}
    \caption{Magnetic field comparison (Observed, baseline, GL-guidance).}
    \label{fig:best_sample_mag}
  \end{subfigure}
  \hfill
  \begin{subfigure}[t]{0.49\textwidth}
    \centering
    \includegraphics[width=\linewidth]{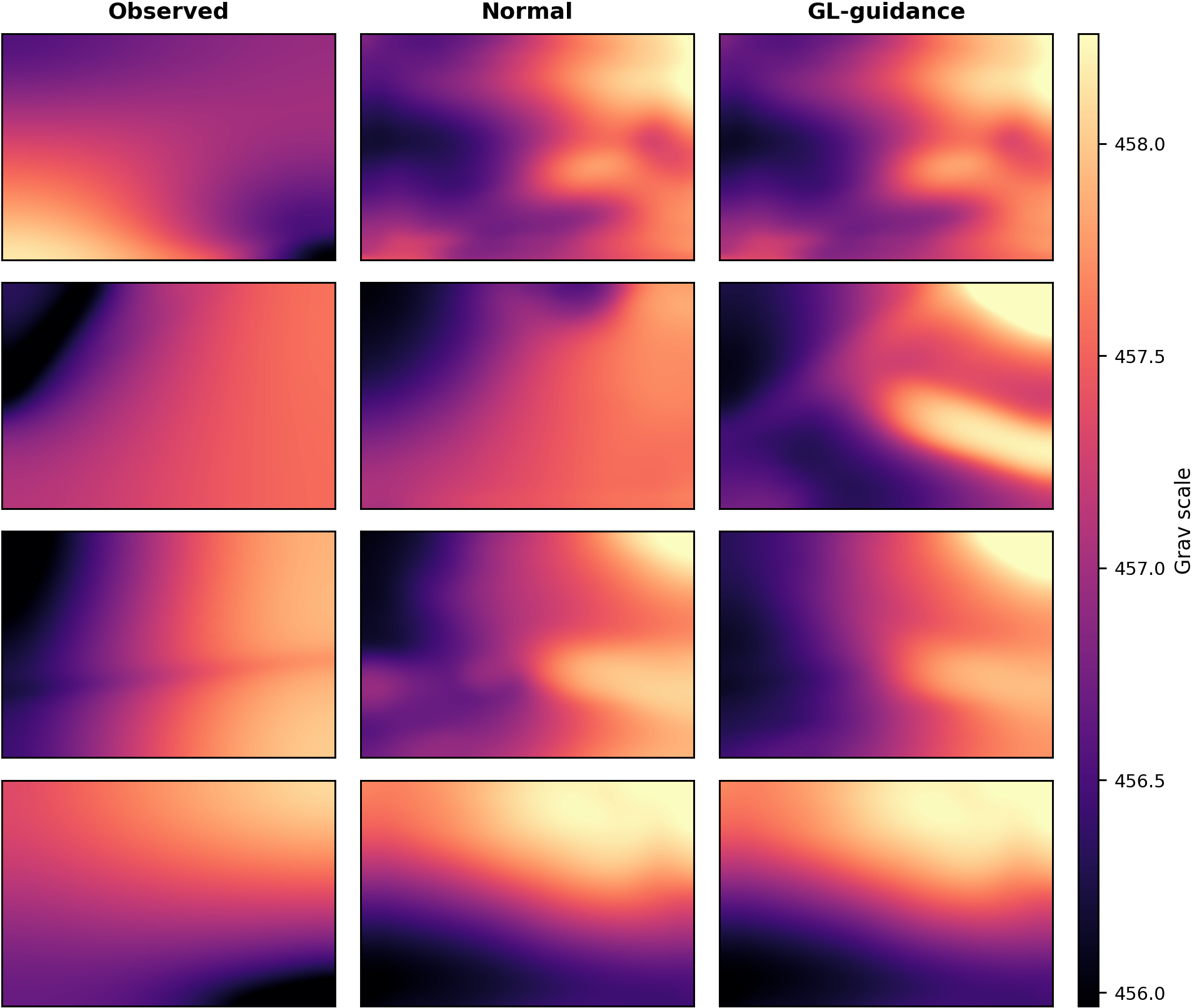}
    \caption{Gravity field comparison (Observed, baseline, GL-guidance).}
    \label{fig:best_sample_grav}
  \end{subfigure}
  \caption{Per-observation qualitative comparison of observed and predicted fields.}
  \label{fig:best_sample_grids}
\end{figure}

\FloatBarrier

For quantitative evaluation, we compute per-sample RMSE for both magnetic and gravity fields relative to the observations. For each observation and metric, methods are ranked by RMSE, and these statistics are aggregated over 128 generated samples to estimate method-level performance.

Figure~\ref{fig:quant_plots}(a) reports the percent improvement in RMSE from the baseline. Across both magnetic and gravity channels, GL guidance shifts the distribution toward positive percent improvement, indicating more reliable reductions in error.

Figure~\ref{fig:quant_plots}(b) shows empirical CDFs of percent RMSE improvement, confirming the same ordering at the distribution level: GL guidance places more mass above zero improvement for both modalities.

\begin{figure*}[htbp]
  \centering
  \begin{subfigure}[t]{0.48\textwidth}
    \centering
    \includegraphics[width=\linewidth]{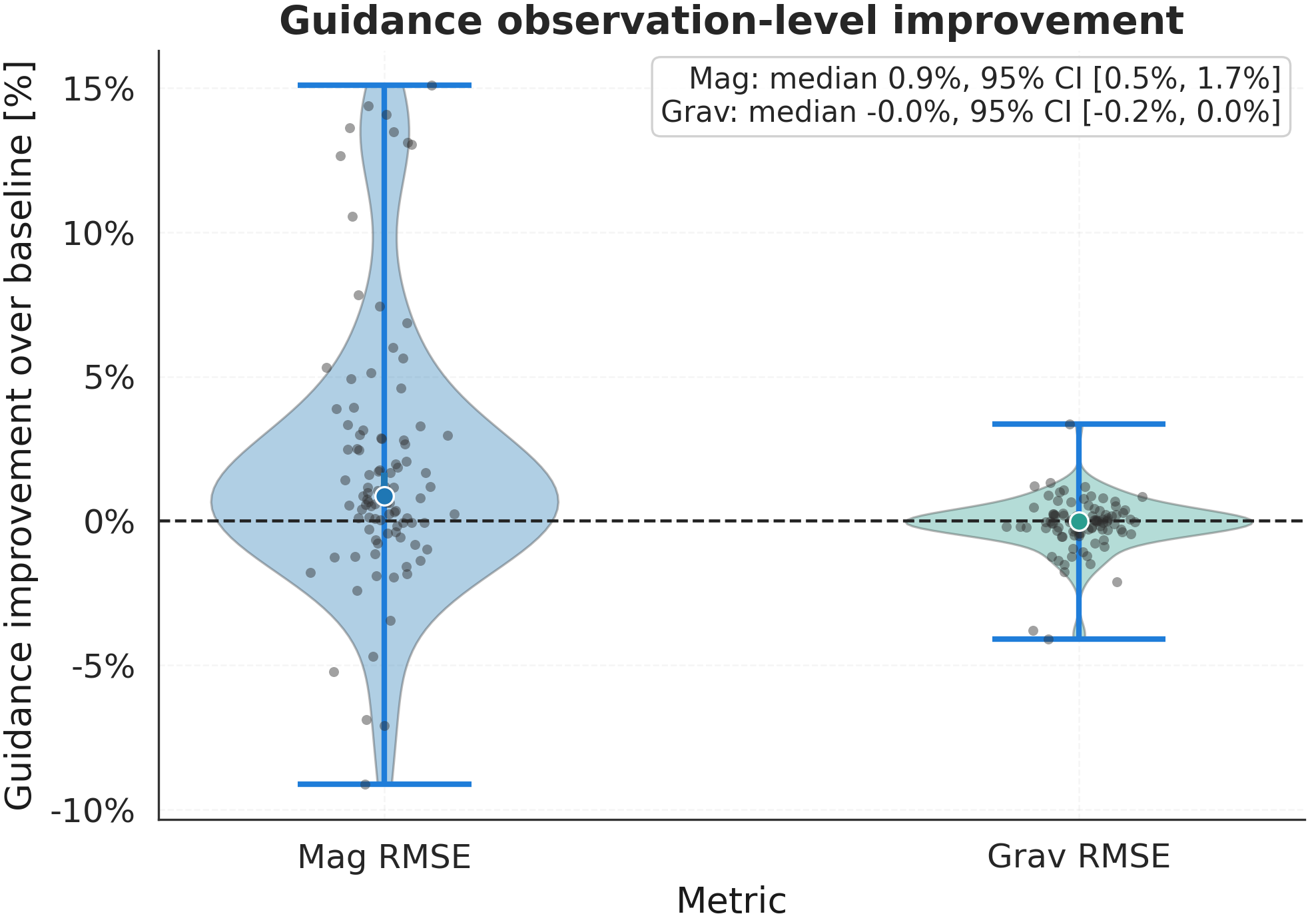}
    \caption{Percent RMSE improvement}
    \label{fig:quant_a}
  \end{subfigure}
  \hfill
  \begin{subfigure}[t]{0.48\textwidth}
    \centering
    \includegraphics[width=\linewidth]{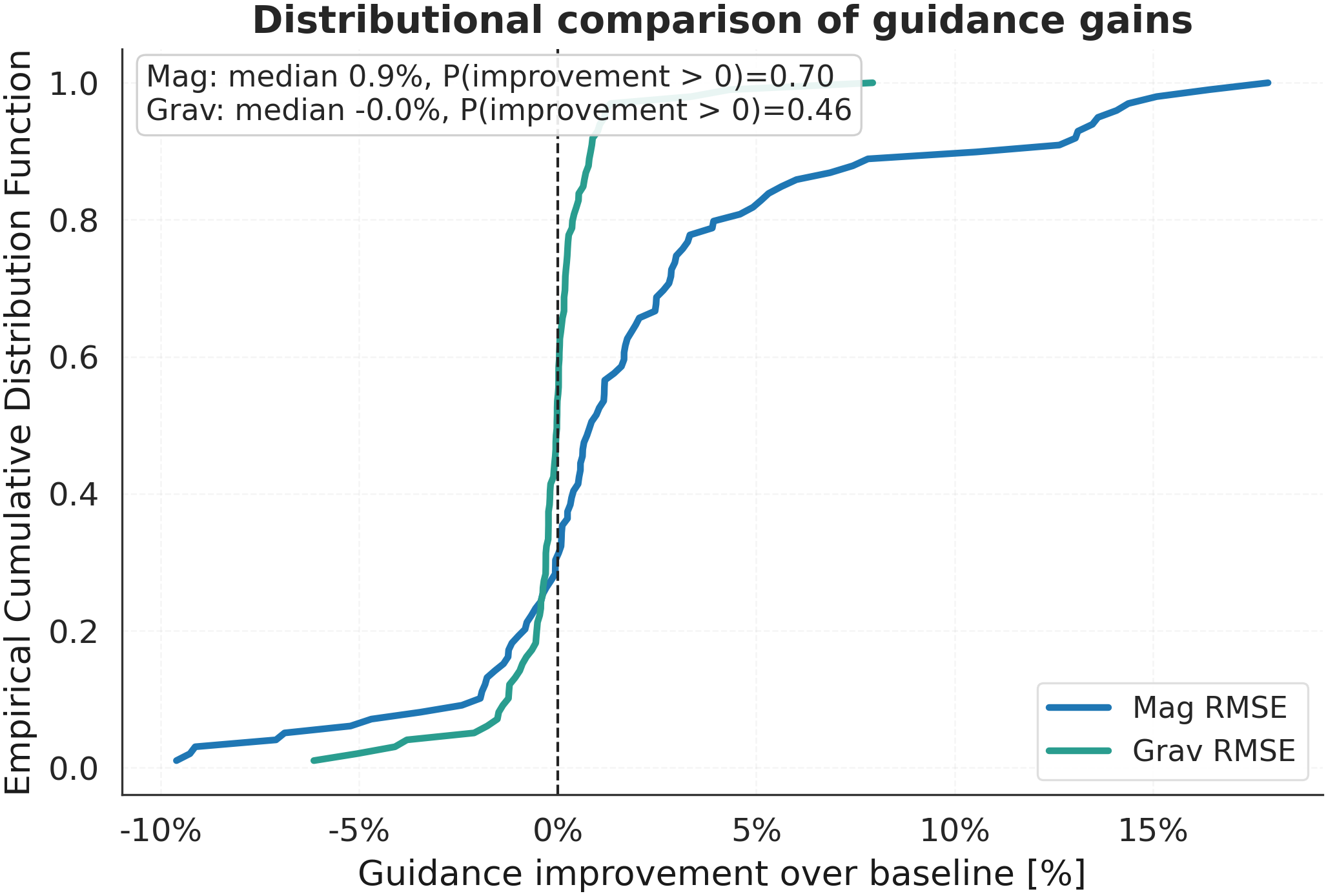}
    \caption{ECDF of percent RMSE improvement}
    \label{fig:quant_b}
  \end{subfigure}

  \caption{Quantitative comparison of decoder-space GL guidance relative to the baseline flow model. Positive percent improvement indicates lower error than the baseline.}
  \label{fig:quant_plots}
\end{figure*}

\section{Computational Complexity of Existing Methods}

\begin{center}
\renewcommand{\arraystretch}{1.35}
\setlength{\tabcolsep}{8pt}
\small
\begin{tabular}{|c|c|c|}
\hline
\textbf{Method} & \textbf{Memory} & \textbf{Time per inversion} \\
\hline
Explicit pseudoinverse / SVD
  & $\mathcal{O}(mn)$ to $\mathcal{O}(n^{2})$
  & $\mathcal{O}(\min(mn^{2},\,m^{2}n))$ \\
Direct Tikhonov via Hessian
  & $\mathcal{O}(n^{2})$
  & $\mathcal{O}(n^{3})$ \\
CGLS / LSQR
  & $\mathcal{O}(n+m)$, matrix-free
  & $\mathcal{O}\!\bigl(k(C_{G}+C_{G^{\top}})\bigr)$ \\
Tikhonov--CG
  & $\mathcal{O}(n+m)$, matrix-free
  & $\mathcal{O}\!\bigl(k(C_{G}+C_{G^{\top}}+C_{L})\bigr)$ \\
Tikhonov parameter sweep
  & same as CG
  & $\mathcal{O}\!\bigl(N_{\lambda}\,k(C_{G}+C_{G^{\top}})\bigr)$ \\
IRLS / focusing
  & same as CG plus weights
  & $\mathcal{O}\!\bigl(K_{\mathrm{out}}\,k(C_{G}+C_{G^{\top}})\bigr)$ \\
TV / ADMM / primal--dual
  & $\mathcal{O}(n+m)$
  & $\mathcal{O}\!\bigl(K(C_{G}+C_{G^{\top}}+C_{D})\bigr)$ \\
Cross-gradient joint inversion
  & $\mathcal{O}(2n+m)$
  & $\mathcal{O}\!\bigl(K_{\mathrm{out}}\,k(C_{G}+C_{G^{\top}}+C_{\mathrm{cg}})\bigr)$ \\
MCMC
  & $\mathcal{O}(n)$ per chain
  & $\mathcal{O}(S\,C_{G})$ \\
\hline
\end{tabular}
\end{center}
In this table, except from MC-MC all methods are deterministic - some of them like the Explicit Pseudo-inverse and Direct Tikhonov are not possible for our data dimension. The only probabilistic method MCMC is not feasible for us as well. \textit{Ours, in comparison, is linear overall}.

\section{Discussion and Future Work}
We develop physics-aware flow-based methodology to jointly invert magnetic and gravitational fields to generate subsurface susceptibility and density. We also contribute a first-of-its-kind VAE and a rectified flow generative model for the Noddyverse dataset. Further work in the field can be:
\begin{itemize}
    \item \textbf{Scaling the latent dimension and parameter count of the VAE.} A larger VAE allows for modeling sharper volumes. This significantly improves the approximation error in the DPS guidance, which assumes a near-perfect decoder, thus enabling higher quality generation.
    \item \textbf{Training with Decoded-Susceptibility GL Regularization.} Training with the GL term applied directly to decoded susceptibility fields remains an important direction once memory-efficient decoder-space regularization is available.
\end{itemize}

All code is publicly available here:\\
\url{https://anonymous.4open.science/r/maggrav-joint-diffusion-DB71}.

\begin{ack}
Use unnumbered first level headings for the acknowledgments. All acknowledgments
go at the end of the paper before the list of references. Moreover, you are required to declare
funding (financial activities supporting the submitted work) and competing interests (related financial activities outside the submitted work).
More information about this disclosure can be found at: \url{https://neurips.cc/Conferences/2026/PaperInformation/FundingDisclosure}.

Do {\bf not} include this section in the anonymized submission, only in the final paper. You can use the \texttt{ack} environment provided in the style file to automatically hide this section in the anonymized submission.
\end{ack}

\bibliographystyle{unsrt}  
\bibliography{main}

@book{blakely1995potential,
  title     = {Potential Theory in Gravity and Magnetic Applications},
  author    = {Blakely, Richard J.},
  year      = {1995},
  publisher = {Cambridge University Press},
  doi       = {10.1017/CBO9780511549816}
}

@book{hinze2013gravity,
  title     = {Gravity and Magnetic Exploration: Principles, Practices, and Applications},
  author    = {Hinze, William J. and von Frese, Ralph R. B. and Saad, Afif H.},
  year      = {2013},
  publisher = {Cambridge University Press},
  doi       = {10.1017/CBO9780511843129}
}

@article{li1996mag3d,
  title   = {3-D inversion of magnetic data},
  author  = {Li, Yaoguo and Oldenburg, Douglas W.},
  journal = {Geophysics},
  year    = {1996},
  volume  = {61},
  number  = {2},
  pages   = {394--408},
  doi     = {10.1190/1.1443968}
}

@article{li1998grav3d,
  title   = {3-D inversion of gravity data},
  author  = {Li, Yaoguo and Oldenburg, Douglas W.},
  journal = {Geophysics},
  year    = {1998},
  volume  = {63},
  number  = {1},
  pages   = {109--119},
  doi     = {10.1190/1.1444302}
}

@article{portniaguine1999focusing,
  title   = {Focusing geophysical inversion images},
  author  = {Portniaguine, O. and Zhdanov, M. S.},
  journal = {Geophysics},
  year    = {1999},
  volume  = {64},
  number  = {3},
  pages   = {874--887},
  doi     = {10.1190/1.1444596}
}

@article{jessell2022noddyverse,
  title   = {Into the Noddyverse: a massive data store of 3D geological models for machine learning and inversion applications},
  author  = {Jessell, Mark and Guo, Jiateng and Li, Yunqiang and Lindsay, Mark and Scalzo, Richard and Giraud, J{\'e}r{\'e}mie and Pirot, Guillaume and Cripps, Ed and Ogarko, Vitaliy},
  journal = {Earth System Science Data},
  year    = {2022},
  volume  = {14},
  number  = {1},
  pages   = {381--392},
  doi     = {10.5194/essd-14-381-2022}
}

@article{saltus2011nonunique,
  title   = {Unique geologic insights from ``non-unique'' gravity and magnetic interpretation},
  author  = {Saltus, Richard W. and Blakely, Richard J.},
  journal = {GSA Today},
  year    = {2011},
  volume  = {21},
  number  = {12},
  pages   = {4--10}
}

@book{tikhonov1977illposed,
  title     = {Solutions of Ill-Posed Problems},
  author    = {Tikhonov, Andrey N. and Arsenin, Vasiliy Y.},
  year      = {1977},
  publisher = {V. H. Winston \& Sons}
}

@book{tarantola2005inverse,
  title     = {Inverse Problem Theory and Methods for Model Parameter Estimation},
  author    = {Tarantola, Albert},
  year      = {2005},
  publisher = {Society for Industrial and Applied Mathematics (SIAM)},
  doi       = {10.1137/1.9780898717921}
}

@article{mosegaard1995montecarlo,
  title   = {Monte Carlo sampling of solutions to inverse problems},
  author  = {Mosegaard, Klaus and Tarantola, Albert},
  journal = {Journal of Geophysical Research: Solid Earth},
  year    = {1995},
  volume  = {100},
  number  = {B7},
  pages   = {12431--12447},
  doi     = {10.1029/94JB03097}
}

@inproceedings{ho2020ddpm,
  title     = {Denoising Diffusion Probabilistic Models},
  author    = {Ho, Jonathan and Jain, Ajay and Abbeel, Pieter},
  booktitle = {Advances in Neural Information Processing Systems (NeurIPS)},
  year      = {2020},
  eprint    = {2006.11239},
  archivePrefix = {arXiv}
}

@inproceedings{rombach2022ldm,
  title     = {High-Resolution Image Synthesis with Latent Diffusion Models},
  author    = {Rombach, Robin and Blattmann, Andreas and Lorenz, Dominik and Esser, Patrick and Ommer, Bj{\"o}rn},
  booktitle = {Proceedings of the IEEE/CVF Conference on Computer Vision and Pattern Recognition (CVPR)},
  year      = {2022},
  eprint    = {2112.10752},
  archivePrefix = {arXiv}
}

@article{allen1979allencahn,
  title   = {A Microscopic Theory for Antiphase Boundary Motion and Its Application to Antiphase Domain Coarsening},
  author  = {Allen, Samuel M. and Cahn, John W.},
  journal = {Acta Metallurgica},
  year    = {1979},
  volume  = {27},
  number  = {6},
  pages   = {1085--1095},
  doi     = {10.1016/0001-6160(79)90196-2}
}

@article{LiOldenburg1996,
  author  = {Li, Yaoguo and Oldenburg, Douglas W.},
  title   = {3-D inversion of magnetic data},
  journal = {Geophysics},
  year    = {1996},
  volume  = {61},
  number  = {2},
  pages   = {394--408}
}

@article{LiOldenburg1998,
  author  = {Li, Yaoguo and Oldenburg, Douglas W.},
  title   = {3-D inversion of gravity data},
  journal = {Geophysics},
  year    = {1998},
  volume  = {63},
  number  = {1},
  pages   = {109--119}
}

@article{LastKubik1983,
  author  = {Last, B. J. and Kubik, K.},
  title   = {Compact gravity inversion},
  journal = {Geophysics},
  year    = {1983},
  volume  = {48},
  number  = {6},
  pages   = {713--721}
}

@article{PortniaguineZhdanov1999,
  author  = {Portniaguine, Oleg and Zhdanov, Michael S.},
  title   = {Focusing geophysical inversion images},
  journal = {Geophysics},
  year    = {1999},
  volume  = {64},
  number  = {3},
  pages   = {874--887},
  doi     = {10.1190/1.1444596}
}

@article{GallardoDelgado2003,
  author  = {Gallardo-Delgado, Luis A. and P{\'e}rez-Flores, Miguel A. and G{\'o}mez-Trevi{\~n}o, Enrique},
  title   = {A versatile algorithm for joint 3D inversion of gravity and magnetic data},
  journal = {Geophysics},
  year    = {2003},
  volume  = {68},
  number  = {3},
  pages   = {949--959}
}

@article{Pilkington2006,
  author  = {Pilkington, Mark},
  title   = {Joint inversion of gravity and magnetic data for two-layer model},
  journal = {Geophysics},
  year    = {2006}
}

@article{FregosoGallardo2009,
  author  = {Fregoso, Esteban and Gallardo, Luis A.},
  title   = {Cross-gradients joint 3D inversion of gravity and magnetic data},
  journal = {Geophysics},
  year    = {2009}
}

@article{Zhdanov2012,
  author  = {Zhdanov, Michael S. and Gribenko, Alexey and Wilson, Gerald},
  title   = {Generalized joint inversion of multimodal geophysical data using Gramian constraints},
  journal = {Geophysical Research Letters},
  year    = {2012}
}

@article{Lin2018,
  author  = {Lin, Weichao and others},
  title   = {Joint multinary inversion of gravity and magnetic data},
  journal = {Geophysical Journal International},
  year    = {2018}
}

@article{Utsugi2025,
  author  = {Utsugi, Mitsuru and others},
  title   = {Joint inversion of magnetic and gravity data using group lasso regularization for extracting common sparse structures},
  journal = {Earth, Planets and Space},
  year    = {2025},
  doi     = {10.1186/s40623-025-02270-1}
}

@article{Vatankhah2023,
  author  = {Vatankhah, Saeid and others},
  title   = {Generalized joint inversion of gravity, gravity gradient tensor and magnetic data using Gramian constraints},
  journal = {Geophysical Journal International},
  year    = {2023}
}

@article{Bosch2006,
  author  = {Bosch, Mario and others},
  title   = {Joint gravity and magnetic inversion in 3D using Monte Carlo methods},
  journal = {Geophysics},
  year    = {2006}
}

@article{Min2024GMNet,
  author  = {Min, Q. and others},
  title   = {GMNet: A Deep Learning Framework for Joint Inversion of Gravity and Magnetic Data},
  journal = {Remote Sensing},
  year    = {2024}
}

@article{Guo2021Noddy,
  author  = {Guo, J. and others},
  title   = {Deep learning inversion of magnetic data based on a geological model: example using Noddy data},
  journal = {Computers \& Geosciences},
  year    = {2021}
}

@article{Song2021SDE,
  author  = {Song, Yang and Sohl-Dickstein, Jascha and Kingma, Diederik P. and Kumar, Abhishek and Ermon, Stefano and Poole, Ben},
  title   = {Score-Based Generative Modeling through Stochastic Differential Equations},
  journal = {arXiv preprint},
  year    = {2021},
  eprint  = {2011.13456},
  archivePrefix = {arXiv},
  primaryClass  = {cs.LG},
  doi     = {10.48550/arXiv.2011.13456}
}

@inproceedings{Karras2022EDM,
  author  = {Karras, Tero and Aittala, Miika and Aila, Timo and Laine, Samuli},
  title   = {Elucidating the Design Space of Diffusion-Based Generative Models},
  booktitle = {Advances in Neural Information Processing Systems (NeurIPS)},
  year    = {2022},
  doi     = {10.48550/arXiv.2206.00364}
}

@article{Liu2022RectifiedFlow,
  author  = {Liu, Xingchao and Gong, Chengyue and Liu, Qiang},
  title   = {Flow Straight and Fast: Learning to Generate and Transfer Data with Rectified Flow},
  journal = {arXiv preprint},
  year    = {2022},
  eprint  = {2209.03003},
  archivePrefix = {arXiv},
  primaryClass  = {cs.LG},
  doi     = {10.48550/arXiv.2209.03003}
}

@article{Chung2023DPS,
  author  = {Chung, Hyungjin and Kim, Jeongsol and McCann, Michael T. and Klasky, Marc L. and Ye, Jong Chul},
  title   = {Diffusion Posterior Sampling for General Noisy Inverse Problems},
  journal = {arXiv preprint},
  year    = {2022},
  eprint  = {2209.14687},
  archivePrefix = {arXiv},
  primaryClass  = {stat.ML},
  doi     = {10.48550/arXiv.2209.14687}
}

@article{Kawar2022DDRM,
  author  = {Kawar, Bahjat and Elad, Michael and Ermon, Stefano and Song, Jiaming},
  title   = {Denoising Diffusion Restoration Models},
  journal = {arXiv preprint},
  year    = {2022},
  eprint  = {2201.11793},
  archivePrefix = {arXiv},
  primaryClass  = {eess.IV},
  doi     = {10.48550/arXiv.2201.11793}
}

@misc{negahdari2025performancemachinelearningmethods,
      title={Performance of Machine Learning Methods for Gravity Inversion: Successes and Challenges}, 
      author={Vahid Negahdari and Shirin Samadi Bahrami and Seyed Reza Moghadasi and Mohammad Reza Razvan},
      year={2025},
      eprint={2510.09632},
      archivePrefix={arXiv},
      primaryClass={physics.geo-ph},
      url={https://arxiv.org/abs/2510.09632}, 
}

@Article{rs16060995,
AUTHOR = {Zhou, Shuai and Wei, Yue and Lu, Pengyu and Yu, Guangrui and Wang, Shuqi and Jiao, Jian and Yu, Ping and Zhao, Jianwei},
TITLE = {A Deep Learning Gravity Inversion Method Based on a Self-Constrained Network and Its Application},
JOURNAL = {Remote Sensing},
VOLUME = {16},
YEAR = {2024},
NUMBER = {6},
ARTICLE-NUMBER = {995},
URL = {https://www.mdpi.com/2072-4292/16/6/995},
ISSN = {2072-4292},
DOI = {10.3390/rs16060995}
}

@article{huang2021deep,
  author  = {Huang, Rui and Liu, Shuang and Qi, Rui and Zhang, Yujie},
  title   = {Deep Learning {3D} Sparse Inversion of Gravity Data},
  journal = {Journal of Geophysical Research: Solid Earth},
  volume  = {126},
  number  = {11},
  pages   = {e2021JB022476},
  year    = {2021},
  month   = nov,
  doi     = {10.1029/2021JB022476},
  url     = {https://doi.org/10.1029/2021JB022476}
}

@article{nagy1966gravitational,
  author  = {Nagy, Dezso},
  title   = {The Gravitational Attraction of a Right Rectangular Prism},
  journal = {Geophysics},
  volume  = {31},
  number  = {2},
  pages   = {362--371},
  year    = {1966},
  doi     = {10.1190/1.1439779}
}

@misc{wang20253dmeddiffusion3dmedical,
      title={3D MedDiffusion: A 3D Medical Latent Diffusion Model for Controllable and High-quality Medical Image Generation}, 
      author={Haoshen Wang and Zhentao Liu and Kaicong Sun and Xiaodong Wang and Dinggang Shen and Zhiming Cui},
      year={2025},
      eprint={2412.13059},
      archivePrefix={arXiv},
      primaryClass={eess.IV},
      url={https://arxiv.org/abs/2412.13059}, 
}

@misc{kim2025flowdpsflowdrivenposteriorsampling,
      title={FlowDPS: Flow-Driven Posterior Sampling for Inverse Problems}, 
      author={Jeongsol Kim and Bryan Sangwoo Kim and Jong Chul Ye},
      year={2025},
      eprint={2503.08136},
      archivePrefix={arXiv},
      primaryClass={cs.CV},
      url={https://arxiv.org/abs/2503.08136}, 
}

@misc{cicek2016unet,
  title = {3D U-Net: Learning Dense Volumetric Segmentation from Sparse Annotation},
  author = {{\"O}zg{\"u}n {\c C}i{\c c}ek and Ahmed Abdulkadir and Soeren S. Lienkamp and Thomas Brox and Olaf Ronneberger},
  year = {2016},
  eprint = {1606.06650},
  archivePrefix = {arXiv},
  primaryClass = {cs.CV},
  url = {https://arxiv.org/abs/1606.06650}
}


\appendix
Note on Compute : All experiments ran on 4 X A6000 with 48 GB VRAM.
\section{Latent-Space GL Regularization Proxy}
\label{sec:latent_gl_proxy_appendix}

The ideal training-time GL regularizer acts on decoded susceptibility fields, but this requires decoder-space GL evaluation at every optimization step. As a memory-feasible ablation, we instead evaluated a latent-grid proxy for the GL term during training. This proxy is not physically equivalent to susceptibility-space GL regularization, because latent coordinates do not preserve the spatial geometry or phase-field meaning of susceptibility. The resulting model performs worse than the baseline in our experiments, supporting the use of decoder-space GL guidance at inference time.

\begin{figure*}[htbp]
  \centering
  \begin{subfigure}[t]{0.49\textwidth}
    \centering
    \includegraphics[width=\linewidth]{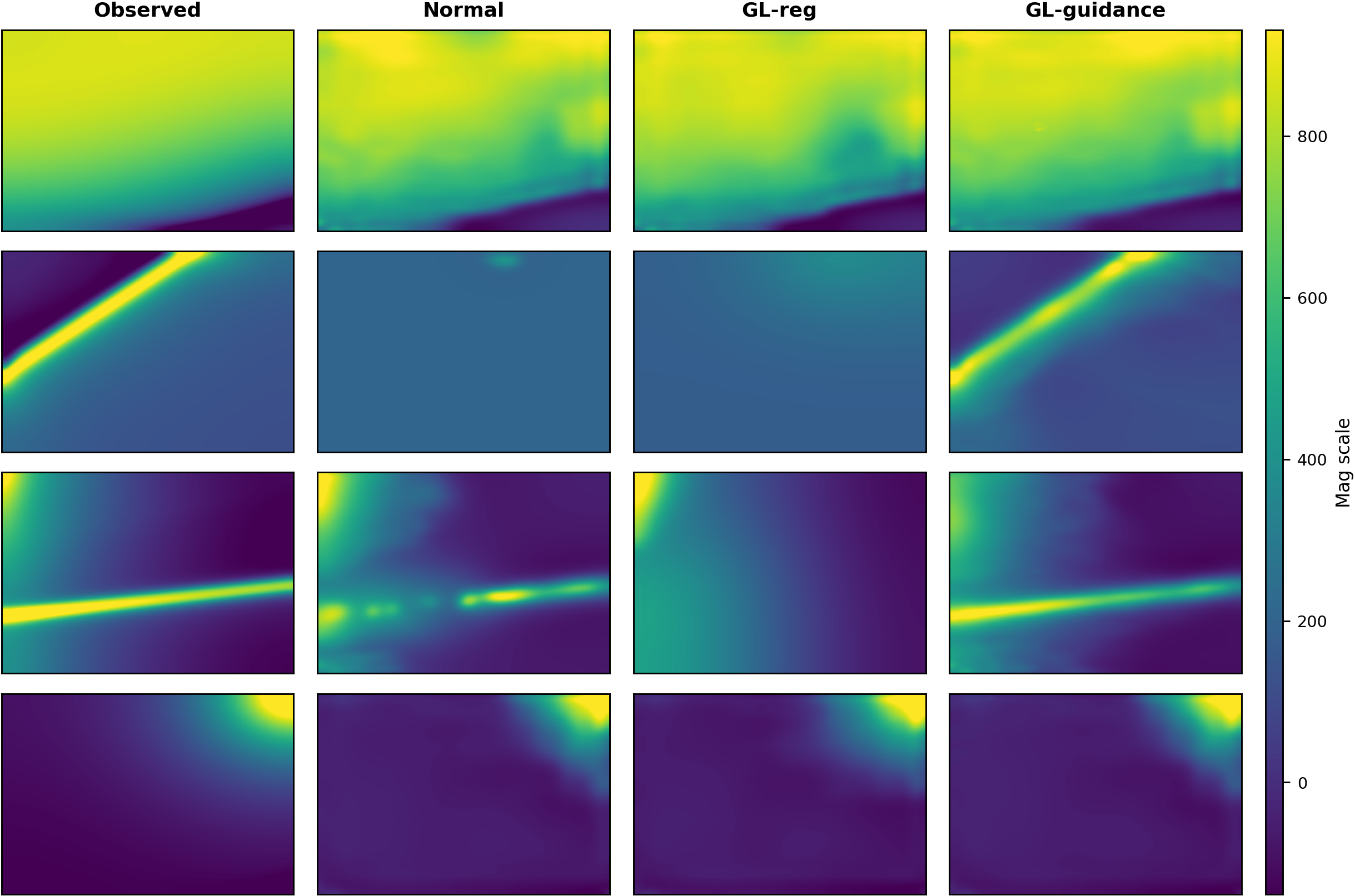}
    \caption{Magnetic field comparison (Observed, baseline, latent GL-reg, GL-guidance).}
    \label{fig:latent_gl_sample_mag}
  \end{subfigure}
  \hfill
  \begin{subfigure}[t]{0.49\textwidth}
    \centering
    \includegraphics[width=\linewidth]{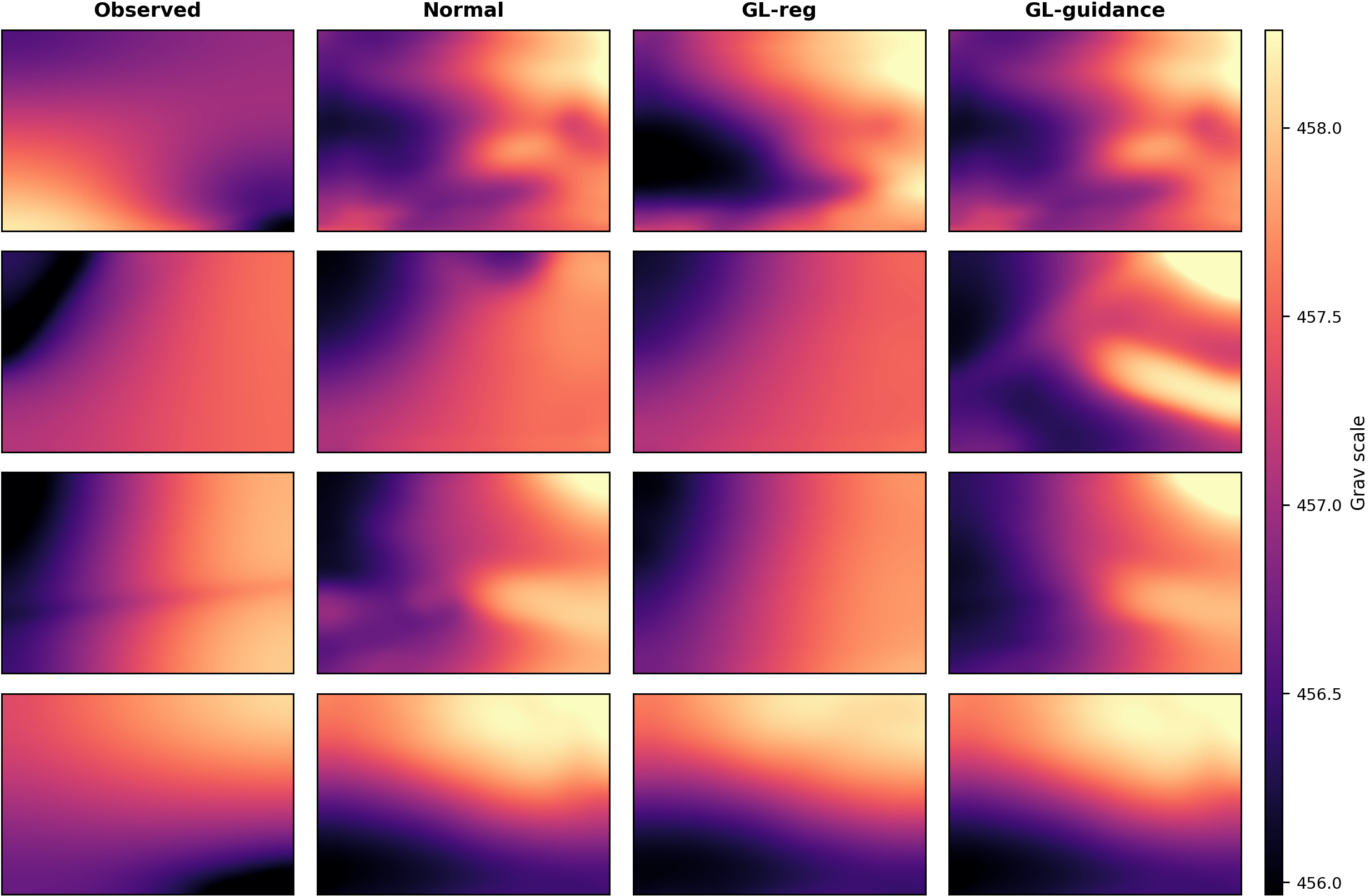}
    \caption{Gravity field comparison (Observed, baseline, latent GL-reg, GL-guidance).}
    \label{fig:latent_gl_sample_grav}
  \end{subfigure}
  \caption{Qualitative comparison including the latent-space GL regularization proxy.}
  \label{fig:latent_gl_sample_grids}
\end{figure*}

\begin{figure*}[htbp]
  \centering
  \begin{subfigure}[t]{0.45\textwidth}
    \centering
    \includegraphics[width=\linewidth]{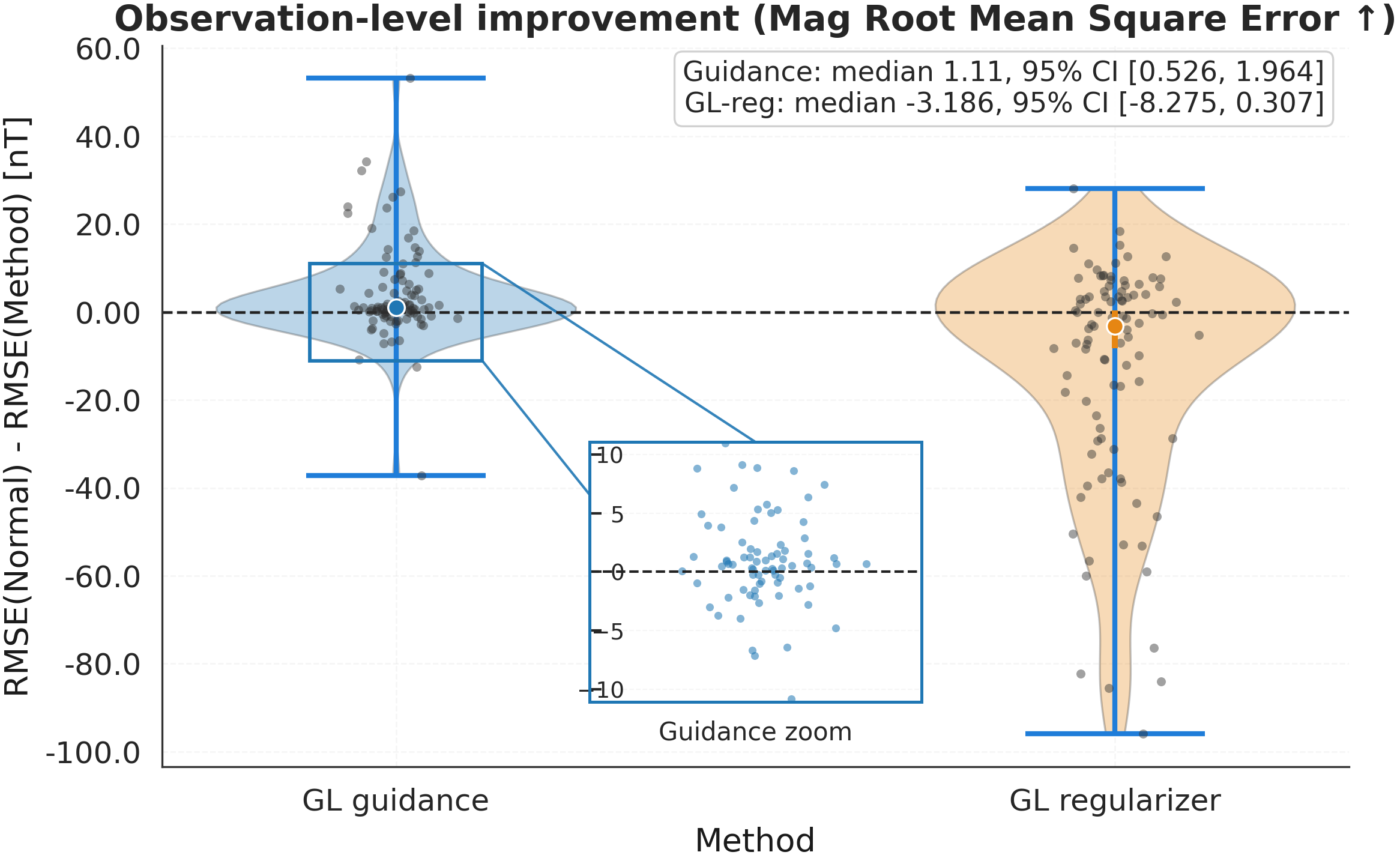}
    \caption{Magnetic percent RMSE improvement}
    \label{fig:latent_gl_delta_mag}
  \end{subfigure}
  \hfill
  \begin{subfigure}[t]{0.45\textwidth}
    \centering
    \includegraphics[width=\linewidth]{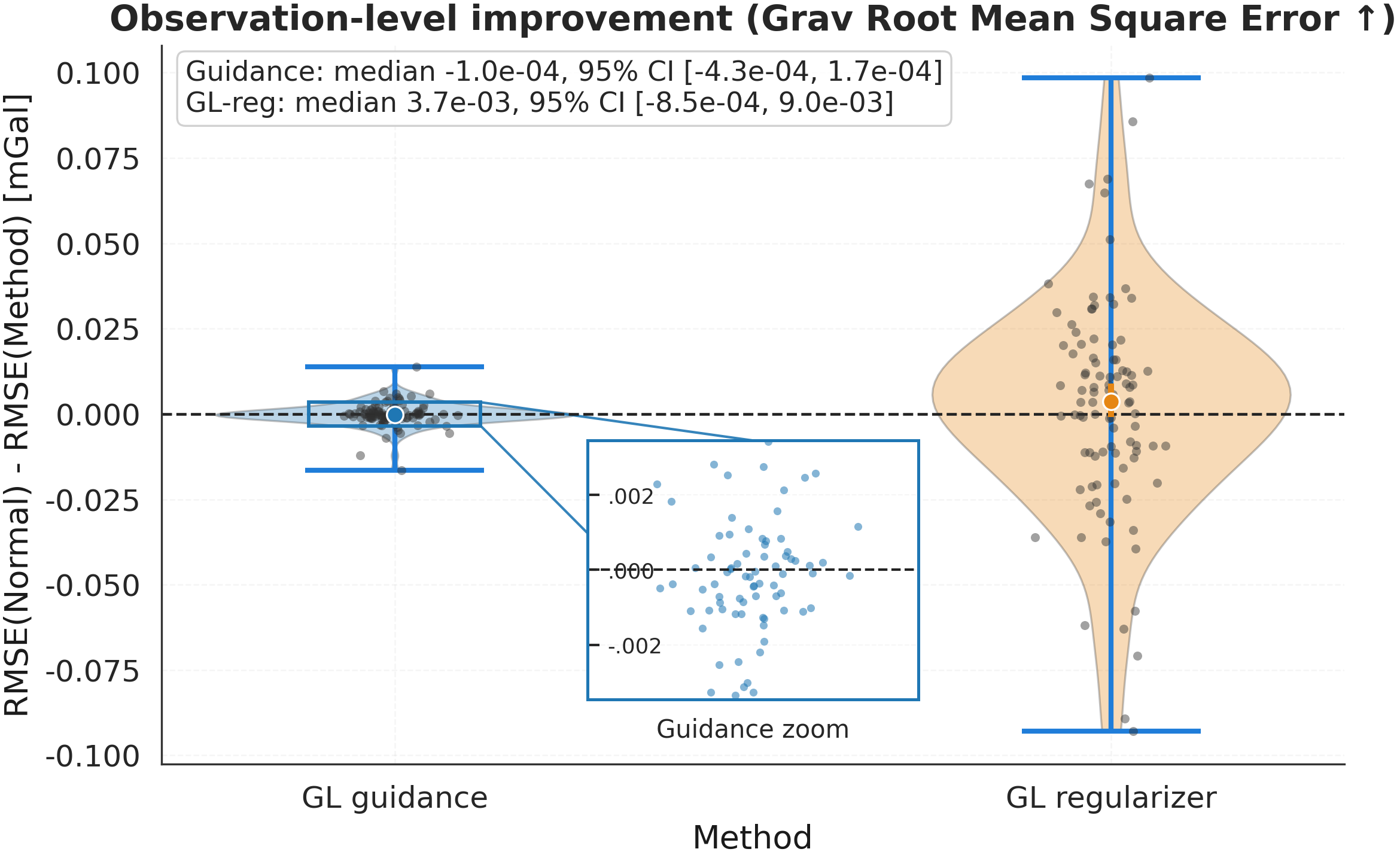}
    \caption{Gravity percent RMSE improvement}
    \label{fig:latent_gl_delta_grav}
  \end{subfigure}
  \caption{Percent RMSE improvement comparison including the latent-space GL regularization proxy. Positive values indicate lower error than the baseline.}
  \label{fig:latent_gl_delta}
\end{figure*}

\begin{figure*}[htbp]
  \centering
  \begin{subfigure}[t]{0.45\textwidth}
    \centering
    \includegraphics[width=\linewidth]{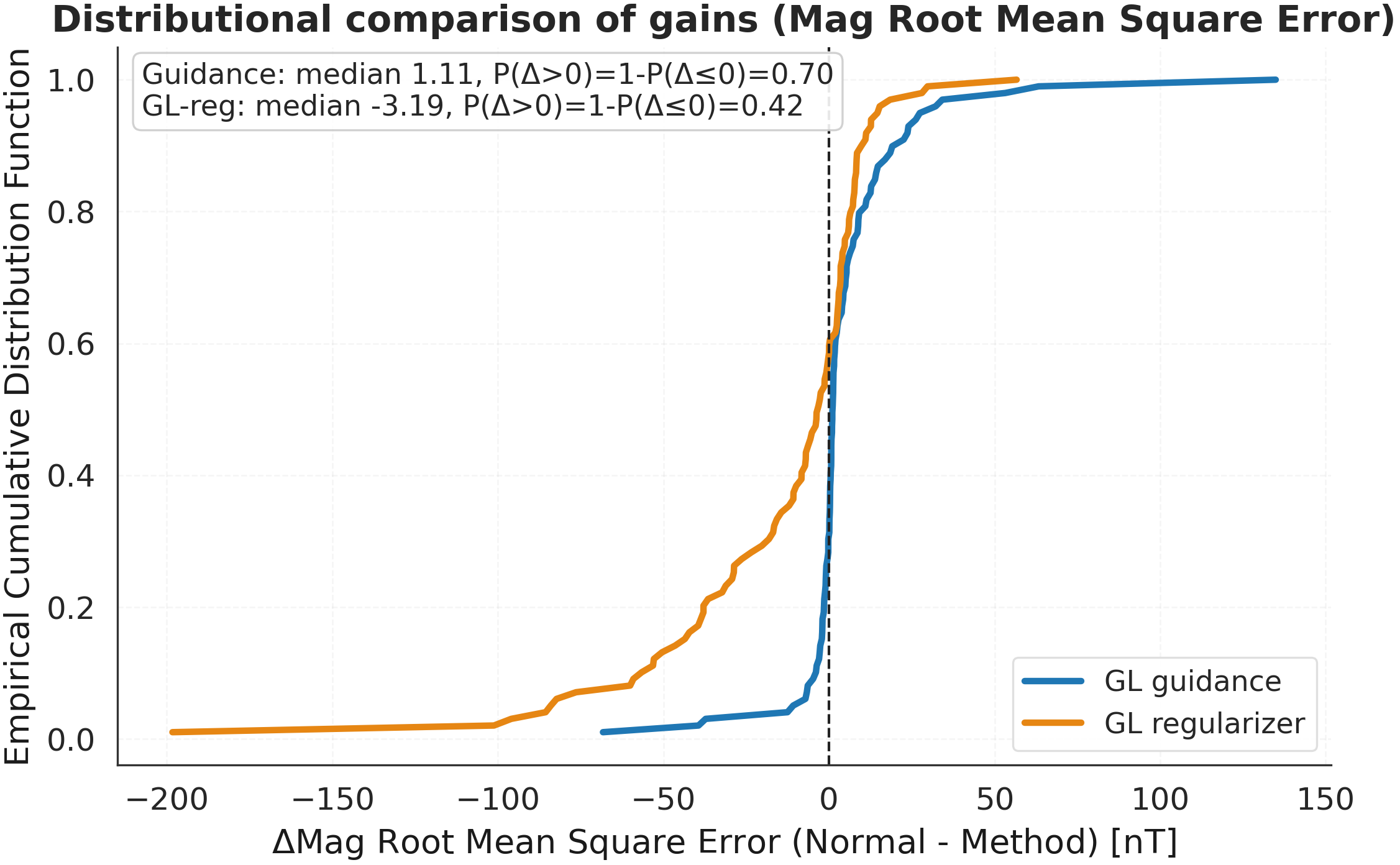}
    \caption{ECDF of percent improvement (magnetic)}
    \label{fig:latent_gl_ecdf_mag}
  \end{subfigure}
  \hfill
  \begin{subfigure}[t]{0.45\textwidth}
    \centering
    \includegraphics[width=\linewidth]{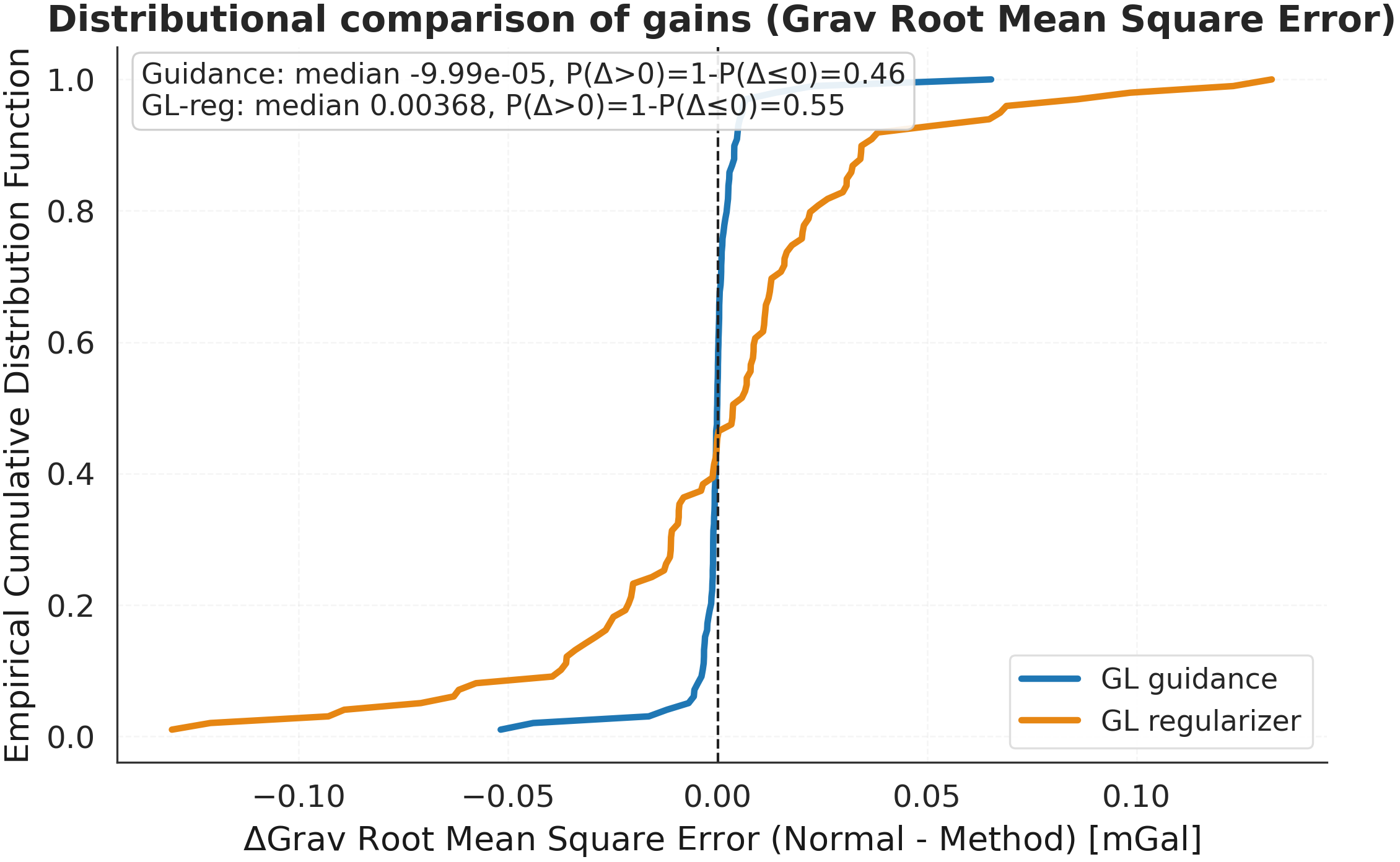}
    \caption{ECDF of percent improvement (gravity)}
    \label{fig:latent_gl_ecdf_grav}
  \end{subfigure}
  \caption{Empirical CDFs of percent RMSE improvement including the latent-space GL regularization proxy.}
  \label{fig:latent_gl_ecdf}
\end{figure*}

\begin{figure*}[htbp]
  \centering
  \begin{subfigure}[t]{0.45\textwidth}
    \centering
    \includegraphics[width=\linewidth]{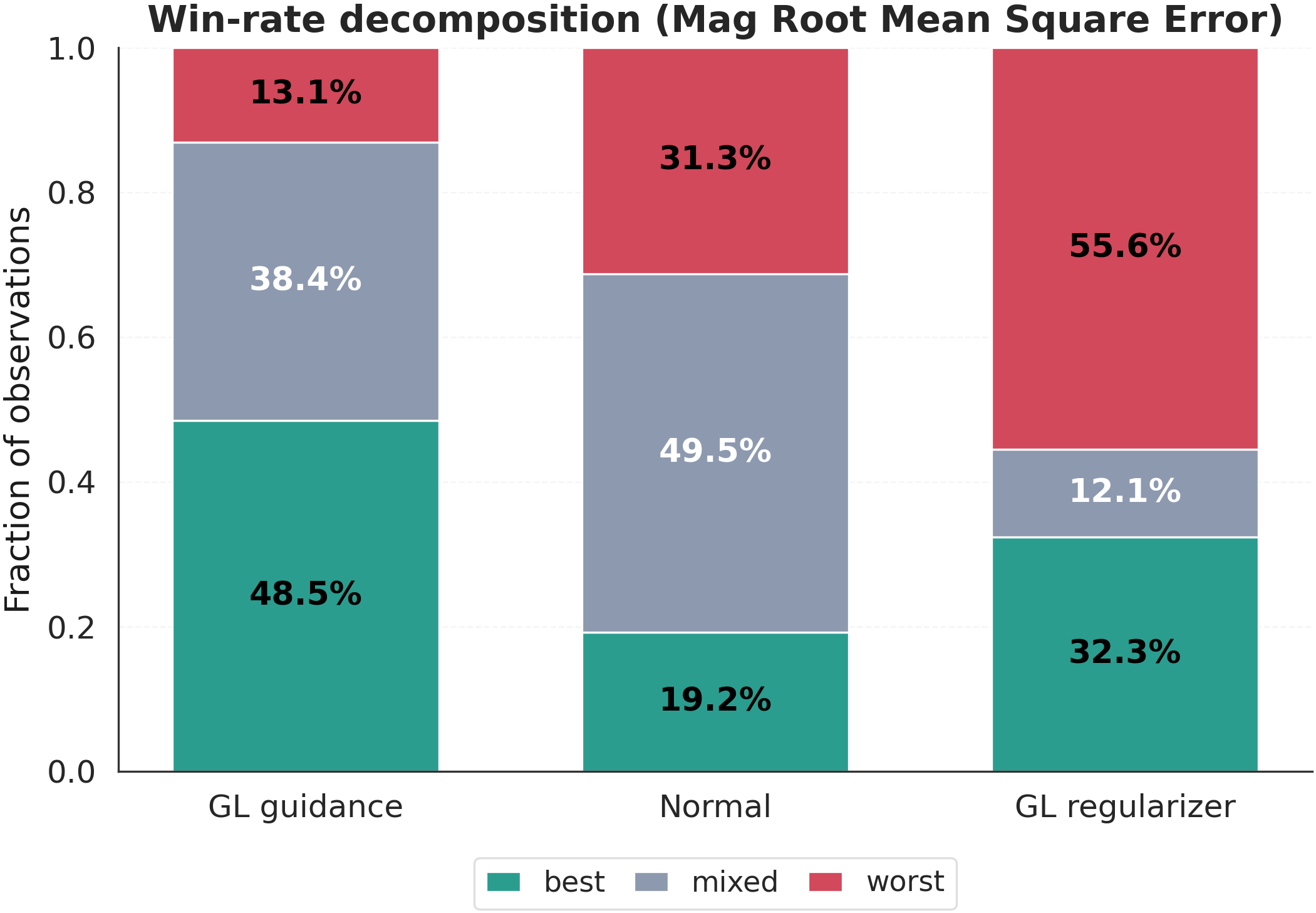}
    \caption{Win-rate (magnetic)}
    \label{fig:latent_gl_winrate_mag}
  \end{subfigure}
  \hfill
  \begin{subfigure}[t]{0.45\textwidth}
    \centering
    \includegraphics[width=\linewidth]{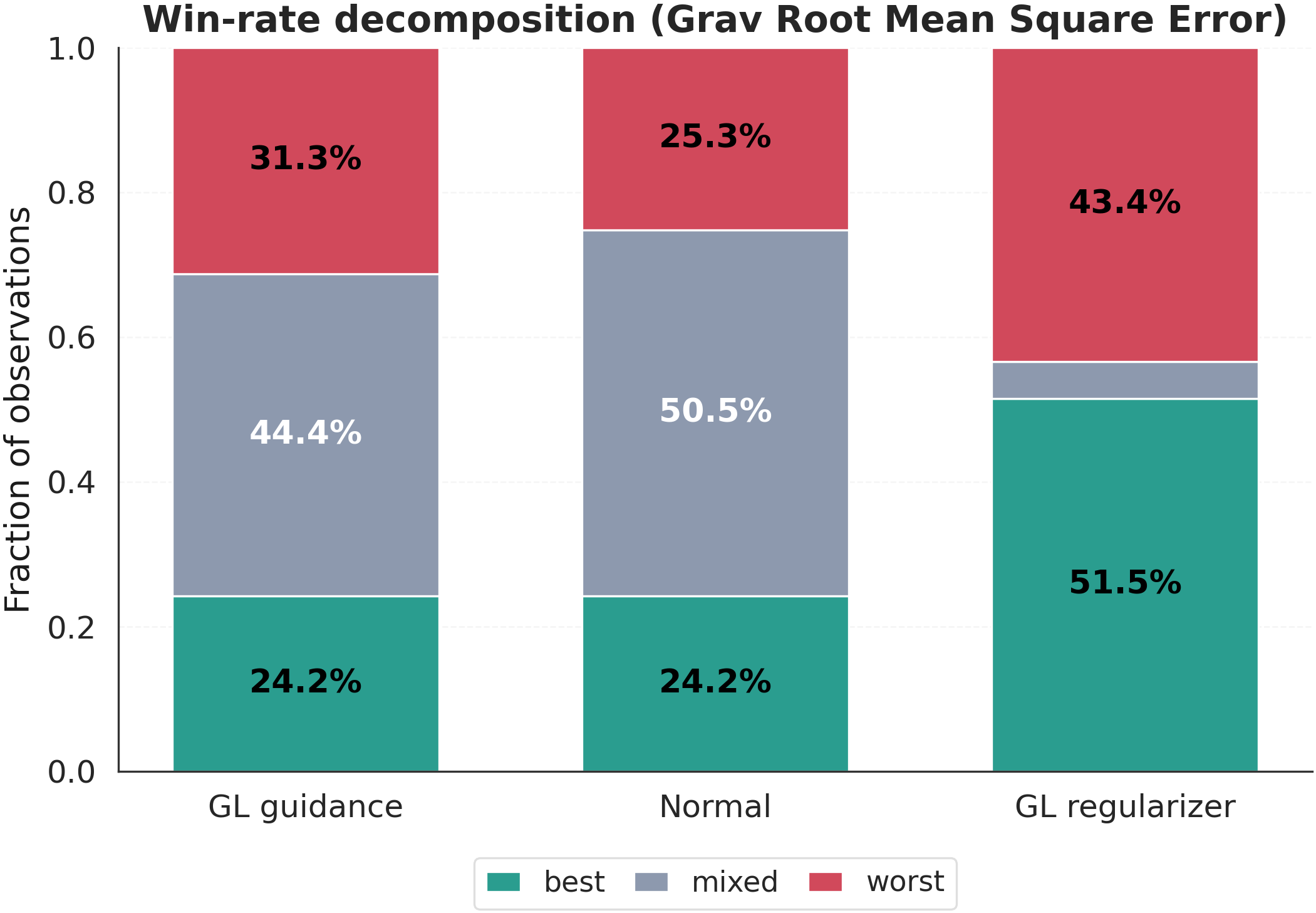}
    \caption{Win-rate (gravity)}
    \label{fig:latent_gl_winrate_grav}
  \end{subfigure}
  \caption{Ranking frequencies including the latent-space GL regularization proxy.}
  \label{fig:latent_gl_winrate}
\end{figure*}

\section{Diffusion and Flow-Based Generative Models}
\label{sec:diffusion_flow_appendix}

Diffusion models define a forward process that gradually perturbs data into noise and learn the corresponding reverse denoising process~\cite{ho2020ddpm}. In DDPM notation, the marginal forward corruption is
\begin{equation*}
q(x_t \mid x_0)=\mathcal{N}\!\left(\sqrt{\bar{\alpha}_t}x_0, (1-\bar{\alpha}_t)\Imat\right).
\end{equation*}
In continuous time, score-based models describe the reverse-time dynamics as~\cite{Song2021SDE}
\begin{equation*}
d x_t = \left[f(x_t,t)-g(t)^2\nabla_{x_t}\log p_t(x_t)\right]dt + g(t)d\bar{w}_t,
\end{equation*}
where \(\bar{w}_t\) is a reverse-time Wiener process and the score \(\nabla_{x_t}\log p_t(x_t)\) is learned by a neural network.

Rectified flow instead learns a deterministic transport from a base distribution to the data distribution~\cite{Liu2022RectifiedFlow}. For endpoints \(x_0\) and \(x_1\), it uses the linear path
\begin{equation*}
x_t=(1-t)x_0+t x_1,
\qquad
v^\star(x_t,t)=x_1-x_0,
\end{equation*}
and trains \(v_\theta\) to approximate this velocity field. Sampling is then performed by integrating the ODE
\begin{equation*}
\frac{d x_t}{dt}=v_\theta(x_t,t).
\end{equation*}
EDM-style continuous-time parameterizations provide related stability and sampling-efficiency improvements for diffusion-family models~\cite{Karras2022EDM}.

\section{Theoretical Proofs}
\subsection{Gravity Prism Kernel}
\label{subsec:gravity_prism}
For cell $V_j$ and observation location $\xvec_i$, the gravity sensitivity element is
\begin{equation*}
[\Gmat_\rho]_{ij} = G \int_{V_j} \frac{z_i - z'}{|\xvec_i - \xvec'|^3} \, d^3\xvec'.
\end{equation*}

\begin{proposition}[Analytical Kernel for Rectangular Prisms]
For a rectangular prism with corners at $(x_1, y_1, z_1)$ and $(x_2, y_2, z_2)$, the gravity kernel element evaluated at observation point $(x_0, y_0, z_0)$ is (see appendix for proof):
\begin{equation}
\resizebox{\linewidth}{!}{$
[\mathbf{G}_\rho]_{ij} = 
G \sum_{k=1}^{2} \sum_{l=1}^{2} \sum_{m=1}^{2} (-1)^{k+l+m}
\Big[
x_k' \ln(y_l' + r_{klm}) +
y_l' \ln(x_k' + r_{klm}) -
z_m' \arctan\!\left(\frac{x_k' y_l'}{z_m' r_{klm}}\right)
\Big]
$}
\label{eq:prism_gravity}
\end{equation}
\end{proposition}

\noindent \text{where} $x_{k}' = x_k - x_0$, $y_{l}' = y_l - y_0$, $z_{m}' = z_m - z_0$, \text{and} $r_{klm} = \sqrt{x_{k}'^2 + y_{l}'^2 + z_{m}'^2}$.\\

\begin{proof} 
The derivation proceeds by direct integration of \eqref{eq:gravity_anomaly}. We have:
\begin{align*}
[\Gmat_\rho]_{ij} &= G \int_{z_1}^{z_2} \int_{y_1}^{y_2} \int_{x_1}^{x_2} \frac{z_0 - z'}{[(x_0-x')^2 + (y_0-y')^2 + (z_0-z')^2]^{3/2}} \, dx' dy' dz'
\end{align*}

\noindent Substituting $u = x' - x_0$, $v = y' - y_0$, $w = z' - z_0$:
\begin{align*}
[\Gmat_\rho]_{ij} &= -G \int_{w_1}^{w_2} \int_{v_1}^{v_2} \int_{u_1}^{u_2} \frac{w}{(u^2 + v^2 + w^2)^{3/2}} \, du\, dv\, dw
\end{align*}

\noindent The inner integral with respect to $u$ yields:
\begin{align*}
\int \frac{w}{(u^2 + v^2 + w^2)^{3/2}} \, du = \frac{wu}{(v^2 + w^2)\sqrt{u^2 + v^2 + w^2}}
\end{align*}

\noindent Continuing with the $v$ and $w$ integrations and applying the fundamental theorem of calculus at each boundary yields the closed-form expression in \eqref{eq:prism_gravity}. [\cite{nagy1966gravitational}
\end{proof}

\subsection{Magnetic Sensitivity Kernel}
\label{subsec:magnetic_kernel}

The magnetic sensitivity kernel $\Gmat_\chi \in \R^{M_m \times N}$ depends on both the geometry and the ambient field direction characterized by inclination $I$ and declination $D$, with $\alpha$ and $\beta$ being angles relating the observation-source geometry to the field direction.

\begin{align*}
[\Gmat_\chi]_{ij} &= \frac{\mu_0 H_0}{4\pi} \int_{V_j} \Bigg[ (3\cos^2\alpha - 1) \frac{1}{r^3} + 3\cos\alpha\cos\beta \frac{3}{r^3} + \ldots \Bigg] d^3\xvec'
\end{align*}

For rectangular prisms, this integral can be evaluated analytically following similar techniques as the gravity kernel, with the geomagnetic field direction incorporated through the trigonometric factors.

\subsection{Graph Laplacian}
\label{subsec:graph_laplacian}
The standard graph Laplacian has elements
\begin{equation*}
L_{ij} = \begin{cases}
|\mathcal{N}(i)| & \text{if } i = j, \\
-1 & \text{if } j \in \mathcal{N}(i), \\
0 & \text{otherwise}.
\end{cases}
\end{equation*}

\subsection{Discrete GL Energy}
\label{subsec:discrete_gl}
Discretizing \eqref{eq:gl_functional} on a grid with $N$ cells yields
\begin{equation*}
\mcE_{\text{GL}}[\phivec] = \sum_{i=1}^{N} \Delta V\left[\frac{\kappa}{4}\sum_{j \in \mathcal{N}(i)} \frac{(\phi_i - \phi_j)^2}{h^2} + \frac{1}{4\epsilon^2}(\phi_i^2 - 1)^2\right],
\end{equation*}
where $\mathcal{N}(i)$ denotes the neighbors of cell $i$, $h$ is the grid spacing, and $\Delta V = h^3$ is the cell volume. Using the graph Laplacian in Appendix~\ref{subsec:graph_laplacian} gives the matrix form in \eqref{eq:gl_matrix}.

\subsection{Ising--Ginzburg--Landau Correspondence}
\label{subsec:ising_gl_correspondence}
Let $\Omega \subset \mathbb{R}^d$ be a bounded Lipschitz domain. Consider
\begin{equation*}
\mathcal{E}_\varepsilon(\phi)
=
\int_\Omega
\left(
\frac{\varepsilon}{2} |\nabla \phi|^2
+
\frac{1}{\varepsilon} W(\phi)
\right) dx,
\end{equation*}
where $W \in C^2(\mathbb{R})$ satisfies $W \ge 0$, $W(\pm 1)=0$, $W(s)>0$ for $s \neq \pm 1$, and $W''(\pm 1)>0$.

\noindent With $c_0 = \int_{-1}^{1} \sqrt{2W(s)} \, ds$, define
\begin{equation*}
\mathcal{E}_0(u)
=
\begin{cases}
c_0 \, \mathrm{Per}(E; \Omega), & u = \chi_E - \chi_{\Omega \setminus E}, \\
+\infty, & \text{otherwise},
\end{cases}
\end{equation*}
for sets $E \subset \Omega$ of finite perimeter.

\begin{theorem}[Ising-GL Correspondence (Modica--Mortola)]
As $\varepsilon \to 0$, $\mathcal{E}_\varepsilon$ $\Gamma$-converges to $\mathcal{E}_0$ in $L^1(\Omega)$.
\end{theorem}

\begin{proof}
We prove compactness, \texttt{liminf} inequality, and recovery sequence construction.

\noindent Assume $\sup_\varepsilon \mathcal{E}_\varepsilon(\phi_\varepsilon) < \infty$. Then
\[
\int_\Omega \frac{1}{\varepsilon} W(\phi_\varepsilon) dx \le C.
\]
\noindent Since $W(s) \gtrsim (|s|-1)^2$ near its wells, it follows that $\phi_\varepsilon \to u$ in $L^1(\Omega)$ with $u(x) \in \{\pm 1\}$ almost everywhere.

\noindent Define
\[
\Phi(s) = \int_0^s \sqrt{2W(r)} dr.
\]
\noindent Using the Modica inequality,
\[
\frac{\varepsilon}{2} |\nabla \phi|^2 + \frac{1}{\varepsilon} W(\phi)
\ge
|\nabla \Phi(\phi)|
\implies
\int_\Omega |\nabla \Phi(\phi_\varepsilon)| dx \le C.
\]

\noindent Hence $\Phi(\phi_\varepsilon)$ is bounded in $BV(\Omega)$ and converges (up to subsequence) to $\Phi(u)$ in $L^1$. Therefore $u \in BV(\Omega; \{\pm1\})$.
\noindent By the Modica inequality,
\[
\mathcal{E}_\varepsilon(\phi_\varepsilon)
\ge
\int_\Omega |\nabla \Phi(\phi_\varepsilon)| dx.
\]
\noindent By lower semicontinuity of the BV norm,
\[
\liminf_{\varepsilon \to 0}
\mathcal{E}_\varepsilon(\phi_\varepsilon)
\ge
\int_\Omega |\nabla \Phi(u)| dx.
\]
\noindent Since $u$ takes only $\pm1$ values,
\[
|\nabla \Phi(u)| = c_0 |\nabla \chi_E|,
\]
\noindent which equals $c_0 \mathrm{Per}(E; \Omega)$.

\noindent Let $u = \chi_E - \chi_{\Omega \setminus E}$.

\noindent Define
\[
\phi_\varepsilon(x)
=
q\left(\frac{d(x,\partial E)}{\varepsilon}\right),
\]
\noindent where $q$ solves the heteroclinic ODE
\[
q' = \sqrt{2W(q)}, \quad q(-\infty) = -1, \quad q(+\infty) = 1.
\]
\noindent Using the coarea formula and one-dimensional energy computation,
\[
\lim_{\varepsilon \to 0}
\mathcal{E}_\varepsilon(\phi_\varepsilon)
=
c_0 \mathrm{Per}(E; \Omega).
\]
\noindent Thus, the limsup inequality holds.

\noindent The $\Gamma$-convergence follows.
\end{proof}

\subsection{Ginzburg--Landau Energy Gradient}
\label{subsec:gl_gradient_appendix}
\begin{proof}
For the gradient term:
\begin{equation*}
\frac{\partial}{\partial \phi_k}\left[\frac{\kappa}{2h^2}\phivec^\top \Lmat \phivec\right] = \frac{\kappa}{h^2}[\Lmat \phivec]_k
\end{equation*}
since $\Lmat$ is symmetric.

\noindent For the double-well term:
\begin{equation*}
\frac{\partial}{\partial \phi_k}\left[\frac{1}{4\epsilon^2}(\phi_k^2 - 1)^2\right] = \frac{1}{\epsilon^2}\phi_k(\phi_k^2 - 1)
\end{equation*}
\end{proof}

\subsection{Ginzburg--Landau Energy Hessian}
The Hessian of the GL energy is
\begin{equation*}
\Hmat_{\text{GL}} = \frac{\kappa}{h^2}\Lmat + \frac{1}{\epsilon^2}\text{diag}(3\phivec \odot \phivec - \bm{1}).
\end{equation*}
\subsection{Allen--Cahn and Gibbs Interpretation}
\label{subsec:allen_cahn}
The GL energy induces the Allen--Cahn gradient flow
\begin{equation*}
\frac{\partial \phi}{\partial t} = -\frac{\delta \mcE_{\text{GL}}}{\delta \phi} = \kappa \nabla^2 \phi - W'(\phi) = \kappa \nabla^2 \phi - \frac{1}{\epsilon^2}\phi(\phi^2 - 1).
\end{equation*}

\begin{lemma}[Energy Dissipation]
The Allen--Cahn equation monotonically decreases the GL energy:
\begin{equation*}
\frac{d\mcE_{\text{GL}}}{dt} = -\int_\Omega \left|\frac{\partial \phi}{\partial t}\right|^2 d\xvec \leq 0.
\end{equation*}
\end{lemma}

\begin{proof}
\begin{align*}
\frac{d\mcE_{\text{GL}}}{dt} &= \int_\Omega \frac{\delta \mcE_{\text{GL}}}{\delta \phi} \frac{\partial \phi}{\partial t} \, d\xvec \\
&= \int_\Omega \frac{\delta \mcE_{\text{GL}}}{\delta \phi} \left(-\frac{\delta \mcE_{\text{GL}}}{\delta \phi}\right) d\xvec \\
&= -\int_\Omega \left|\frac{\delta \mcE_{\text{GL}}}{\delta \phi}\right|^2 d\xvec \leq 0.
\end{align*}
\end{proof}

\noindent Let $\Omega \subset \mathbb{R}^d$ be bounded with Neumann boundary conditions. Consider the stochastic Allen--Cahn equation~\cite{allen1979allencahn}
\begin{equation}
d\phi_t
=
\left(
\kappa \Delta \phi_t - W'(\phi_t)
\right) dt
+
\sqrt{2T} \, dW_t,
\label{eq:stochastic_allen_cahn_rigorous}
\end{equation}
where $W_t$ is a cylindrical Wiener process in $L^2(\Omega)$.

\begin{theorem}[Invariant Gibbs Measure]
Assume suitable coercivity of $W$ and well-posedness of \eqref{eq:stochastic_allen_cahn_rigorous}. Then the measure
\begin{equation*}
\mu(d\phi)
=
Z^{-1}
\exp\left(
-\frac{\mathcal{E}_{\text{GL}}[\phi]}{T}
\right)
\, d\phi
\end{equation*}
is invariant for the stochastic Allen--Cahn equation.
\end{theorem}

\begin{proof}
The drift is the negative variational derivative $\frac{\delta \mathcal{E}_{\text{GL}}}{\delta \phi} = -\kappa \Delta \phi + W'(\phi)$. Thus, the SPDE can be written as $d\phi_t = - \frac{\delta \mathcal{E}_{\text{GL}}}{\delta \phi} dt + \sqrt{2T} \, dW_t$.
The formal generator is $\mathcal{L}F(\phi) = \langle -\delta \mathcal{E}_{\text{GL}}/\delta \phi, DF(\phi) \rangle + T \mathrm{Tr}(D^2F(\phi))$.
By integration by parts in infinite dimensions, the adjoint operator satisfies $\mathcal{L}^* \mu = 0$. Thus, $\mu$ is invariant.
\end{proof}

\begin{remark}[GL Weight and Temperature]
The Gibbs measure $p_{\text{eq}}[\phi] \propto \exp\left(-\mathcal E_{\text{GL}}[\phi]/T\right)$ coincides with a GL-regularized prior $p_{\text{GL}}(\phi) \propto \exp(-\lambda_{\text{GL}} \mathcal E_{\text{GL}}[\phi])$ under the identification $\lambda_{\text{GL}}=1/T$.
This identification applies to the stationary Gibbs measure of the SPDE. The inference-time weight $\lambda_{\text{GL}}(t)$ is instead a continuation schedule for reverse diffusion, not a physical temperature evolution.
\end{remark}

\section{FlowDPS Update Details}
\label{sec:flowdps_appendix}
\subsection{Endpoint Refinement and Propagation}
\label{subsec:flowdps_update}
Let \(v_\theta(m_t,t)\) denote the flow velocity. Following FlowDPS, we form linear-flow endpoint estimates:
\[
\hat{m}_0 = m_t - t\,v_\theta(m_t,t),\qquad
\hat{m}_1 = m_t + (1-t)\,v_\theta(m_t,t).
\]
The clean endpoint is refined by a data-consistency step,
\[
m_0^{\mathrm{ref}} = \hat{m}_0 - \alpha_t \nabla_{\hat{m}_0}\,\mathcal{L}_{\mathrm{data}}(\hat{m}_0),
\]
where \(\mathcal{L}_{\mathrm{data}}\) is defined in Section~\ref{sec:dps}. The refined endpoint is blended and propagated via FlowDPS controls:
\[
\tilde{m}_0=(1-\gamma_t)\hat{m}_0+\gamma_t m_0^{\mathrm{ref}}, \qquad
\tilde{m}_1=\sqrt{1-\eta_t}\,\hat{m}_1+\sqrt{\eta_t}\,\epsilon,
\]
followed by the update to \(t_{\text{next}}\).

\section{VAE Training Details and Sweeps}

All VAE experiments were trained on joint density--susceptibility volumes using AdamW with learning rate $10^{-4}$.
Across all sweeps, we used bf16 mixed precision, gradient clipping of $1.0$, random seed $42$, and evaluated every $100{,}000$ processed samples.
Training ran for up to $10^6$ processed samples.
The first two sweeps used a latent spatial size of $40^3$ and a shallower architecture, while the third sweep used a deeper attention VAE with a smaller latent spatial size of $24^3$.

\paragraph{Sweep design.}
Sweep 1 varied only the KL weight while using a relatively large model.
Sweep 2 again varied the KL weight, but reduced the model width and parameter count.
Sweep 3 changed both the latent channel count and KL weight, while also switching to a deeper architecture with two residual blocks per stage and two bottleneck attention blocks.

\begin{table}[t]
\centering
\setlength{\tabcolsep}{4pt}
\caption{VAE sweep configurations. The latent shape is reported as channels $\times$ spatial size$^3$.}
\label{tab:vae_sweep_configs}
\begin{tabular}{lccccccc}
\toprule
Sweep & KL weight & Latent shape & Base ch. & Bottleneck & Res. blocks & Attn. blocks & Params \\
\midrule
1 & $10^{-5}$ & $8 \times 40^3$ & 8 & 48 & 1 & 1 & 614,091 \\
1 & $10^{-4}$ & $8 \times 40^3$ & 8 & 48 & 1 & 1 & 614,091 \\
1 & $10^{-3}$ & $8 \times 40^3$ & 8 & 48 & 1 & 1 & 614,091 \\
\midrule
2 & $10^{-3}$ & $4 \times 40^3$ & 4 & 32 & 1 & 1 & 235,135 \\
2 & $5\times10^{-3}$ & $4 \times 40^3$ & 4 & 32 & 1 & 1 & 235,135 \\
2 & $10^{-2}$ & $4 \times 40^3$ & 4 & 32 & 1 & 1 & 235,135 \\
\midrule
3 & $10^{-3}$ & $2 \times 24^3$ & 4 & 32 & 2 & 2 & 385,578 \\
3 & $5\times10^{-3}$ & $2 \times 24^3$ & 4 & 32 & 2 & 2 & 385,578 \\
3 & $10^{-3}$ & $4 \times 24^3$ & 4 & 32 & 2 & 2 & 385,736 \\
3 & $5\times10^{-3}$ & $4 \times 24^3$ & 4 & 32 & 2 & 2 & 385,736 \\
\bottomrule
\end{tabular}
\end{table}

Table~\ref{tab:vae_sweep_configs} shows that the experiments span a clear reconstruction--regularization--compute tradeoff.
Sweeps 1 and 2 use larger latent spatial grids ($40^3$), whereas Sweep 3 reduces the latent spatial resolution to $24^3$ in order to make downstream latent generative modeling substantially more tractable.

\begin{table}[t]
\centering
\caption{Best validation metrics for all completed VAE sweep runs. We report the checkpoint achieving the best validation reconstruction loss, along with the corresponding validation KL and total loss.}
\label{tab:vae_sweep_results}
\begin{tabular}{lccccc}
\toprule
Sweep & KL weight & Latent shape & Best val recon & Val KL & Val total \\
\midrule
1 & $10^{-5}$ & $8 \times 40^3$ & $7.3\times10^{-3}$ & 7.2 & $7.4\times10^{-3}$ \\
1 & $10^{-4}$ & $8 \times 40^3$ & $7.3\times10^{-3}$ & 3.7 & $7.7\times10^{-3}$ \\
1 & $10^{-3}$ & $8 \times 40^3$ & $8.0\times10^{-3}$ & 1.2 & $9.2\times10^{-3}$ \\
\midrule
2 & $10^{-3}$ & $4 \times 40^3$ & $2.0\times10^{-2}$ & 2.5 & $2.2\times10^{-2}$ \\
2 & $5\times10^{-3}$ & $4 \times 40^3$ & $2.4\times10^{-2}$ & 0.98 & $2.9\times10^{-2}$ \\
2 & $10^{-2}$ & $4 \times 40^3$ & $2.3\times10^{-2}$ & 0.57 & $2.9\times10^{-2}$ \\
\midrule
3 & $10^{-3}$ & $2 \times 24^3$ & $3.9\times10^{-2}$ & 4.0 & $4.3\times10^{-2}$ \\
3 & $5\times10^{-3}$ & $2 \times 24^3$ & $3.7\times10^{-2}$ & 1.6 & $4.5\times10^{-2}$ \\
3 & $10^{-3}$ & $4 \times 24^3$ & $2.9\times10^{-2}$ & 3.0 & $3.2\times10^{-2}$ \\
3 & $5\times10^{-3}$ & $4 \times 24^3$ & $3.5\times10^{-2}$ & 1.3 & $4.2\times10^{-2}$ \\
\bottomrule
\end{tabular}
\end{table}

The validation KL is important beyond the total loss alone, since the latent prior must be regular enough for a downstream diffusion or flow model to learn a useful distribution.
Very small KL weights improve reconstruction, but they also leave a much less regularized latent space.
For example, in Sweep 1, moving from KL weight $10^{-3}$ to $10^{-5}$ improves validation reconstruction from $0.007987$ to $0.007340$, but increases the validation KL from $1.229854$ to $7.163707$.
Conversely, increasing the KL weight reduces the latent KL substantially, but comes at the cost of worse reconstruction.

For the main experiments, we select the Sweep 3 model with latent shape $4 \times 24^3$ and KL weight $10^{-3}$.
This model does not achieve the best raw reconstruction metrics, but it provides a much smaller latent representation for downstream generative modeling while maintaining a moderate KL penalty and acceptable reconstruction quality.
It is clear from the sweep results that higher parameter sizes and larger latent dimensions achieve lower reconstruction losses and better latent regularization as well.
Even though we choose a model with a lower latent dimension for computational reasons, these sweep results make clear that with more compute the VAE quality, and therefore the final inversion pipeline, can likely be improved significantly.

\section{BiFlowNet Backbone Comparison}

We also compared two BiFlowNet backbones for latent rectified-flow training using the selected VAE latents from Sweep 3.
Both models were trained on the same latent dataset with latent shape $4 \times 24^3$, batch size 32, learning rate $5 \times 10^{-5}$, warmup over 150{,}000 samples, and a total training budget of $5 \times 10^6$ processed samples.
The two runs differ only in model capacity.

\begin{table}[t]
\centering
\caption{BiFlowNet backbone comparison on the selected latent representation. Both models were trained on latents encoded with the Sweep 3 VAE checkpoint ($4 \times 24^3$, KL weight $10^{-3}$).}
\label{tab:biflownet_comparison}
\begin{tabular}{lcccccc}
\toprule
Model & Model dim & Mid DiT blocks & Attn. heads & Latent shape & Params & Final val loss \\
\midrule
S & 32 & 1 & 4 & $4 \times 24^3$ & 8,397,188 & 0.215393 \\
L & 48 & 3 & 12 & $4 \times 24^3$ & 19,957,636 & 0.181192 \\
\bottomrule
\end{tabular}
\end{table}

The corresponding final training losses were $0.208852$ for the smaller model and $0.176247$ for the larger model.
The larger BiFlowNet improves both optimization and generalization, achieving a lower training loss as well as a lower validation loss.
This indicates that the larger backbone is not merely overfitting, but is learning a better latent generative prior.

The BiFlowNet comparison mirrors the VAE sweep behavior: increasing model capacity improves performance when compute allows it.
Accordingly, we use the larger BiFlowNet backbone in the main experiments.
At the same time, the gap between the two models suggests that additional scaling of the latent generative model remains a promising direction for improving inversion quality further.

\section{Probabilistic Nature of the Posterior Samples}

Figure~\ref{fig:probabilistic_outputs} highlights the probabilistic nature of our inversion framework.
For a fixed pair of observed magnetic and gravity fields, the model generates multiple 3D susceptibility and density volumes that differ substantially while remaining conditioned on the same 2D observations.
This behavior reflects the severe non-uniqueness of potential-field inversion: multiple subsurface configurations can explain nearly identical surface measurements.
Unlike deterministic inversion methods, our approach preserves this ambiguity by sampling diverse but plausible posterior solutions.

\begin{figure}[t]
    \centering
    \IfFileExists{./plots/probabilistic_plot.png}{%
    \includegraphics[width=\textwidth]{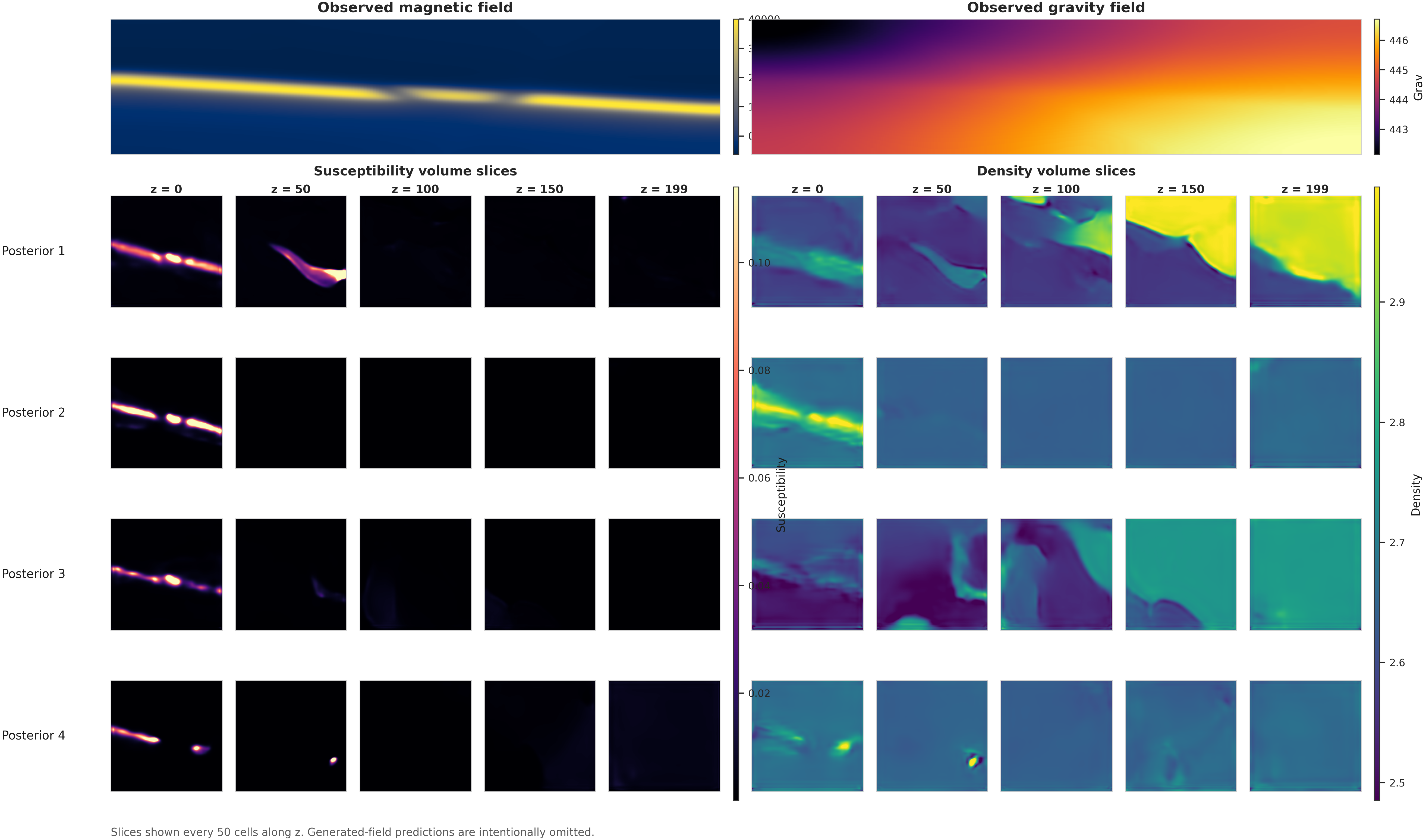}%
    }{%
    \fbox{\parbox{0.98\textwidth}{\centering Missing figure: \texttt{./plots/probabilistic\_plot.png}}}%
    }
    \caption{Posterior diversity under fixed conditioning. The top row shows the observed magnetic and gravity fields. The remaining rows show multiple generated 3D susceptibility and density volume slices for the same observation. Despite identical conditioning, the recovered 3D volumes differ substantially, illustrating the probabilistic nature of our method.}
    \label{fig:probabilistic_outputs}
\end{figure}


\newpage

\end{document}